\documentclass{article}
\usepackage{amssymb}

\usepackage[utf8]{inputenc}
\usepackage{mathtools}
\usepackage{amsmath}
\usepackage{graphicx}
\usepackage{amsthm}
\usepackage{biblatex}
\usepackage{graphicx}
\usepackage[switch]{lineno}
\usepackage{footnote}
\usepackage{authblk}

\usepackage{amsmath}
\usepackage{amssymb}
\usepackage{mathtools}
\usepackage{amsthm}
\usepackage{wrapfig}
\usepackage{mwe}
\usepackage{graphicx}
\usepackage{multirow}
\usepackage{booktabs}
\usepackage{subfig}
\usepackage{algorithm}
\usepackage{algorithmic}

\theoremstyle{plain}
\newtheorem{theorem}{Theorem}[section]

\newtheorem{lemma}[theorem]{Lemma}

\theoremstyle{definition}
\newtheorem{definition}[theorem]{Definition}
\newtheorem{assumption}[theorem]{Assumption}
\theoremstyle{remark}
\newtheorem{remark}[theorem]{Remark}

\title{On the Convergence of Clustered Federated Learning}

\author{Jie MA$^{1}$, Guodong Long$^{1}$, Tianyi Zhou$^{2,3}$, Jing Jiang$^1$, Chengqi Zhang$^1$\\
	$^1$Australian Artificial Intelligence Institute, University of Technology Sydney\\
	$^2$University of Washington, Seattle, $^3$University of Maryland, College Park\\
}

\begin{document}

\maketitle

\begin{abstract}
	Knowledge sharing and model personalization are essential components to tackle the non-IID challenge in federated learning (FL). Most existing FL methods focus on two extremes: 1) to learn a shared model to serve all clients with non-IID data, and 2) to learn personalized models for each client, namely personalized FL. There is a trade-off solution, namely clustered FL or cluster-wise personalized FL, which aims to cluster similar clients into one cluster, and then learn a shared model for all clients within a cluster. This paper is to revisit the research of clustered FL by formulating them into a bi-level optimization framework that could unify existing methods. We propose a new theoretical analysis framework to prove the convergence by considering the clusterability among clients. In addition, we embody this framework in an algorithm, named \textbf{We}ighted \textbf{C}lustered \textbf{F}ederated \textbf{L}earning (WeCFL). Empirical analysis verifies the theoretical results and demonstrates the effectiveness of the proposed WeCFL under the proposed cluster-wise non-IID settings.
	
	\looseness=-1
\end{abstract}


\section{Introduction}
Since Federated Learning (FL)\cite{mcmahan2017communication} was proposed firstly in 2017, it has evolved into a new-generation collaborative machine learning framework with applications in a range of scenarios, including Google's Gboard on Android \cite{mcmahan2017communication}, Apple's Siri \cite{apple2017federated}, Computer Visions \cite{luo2019real,jallepalli2021federated,he2021fedcv}, Smart Cities \cite{zheng2022applications} and Healthcare \cite{rieke2020future,xu2021federated,long2022federated}. 
The vanilla FL method, known as FedAvg \cite{mcmahan2017communication}, is derived from a distributed machine learning framework before it is applied to a large-scale mobile service system. In particular, it aims to train a single shared model at the server by aggregating the smartphones' local model, trained with its own data. Thus, the end users' private data in each smartphone will not be uploaded to the cloud server. FedAvg first proposed the non-IID problem in FL that the data distribution varied across clients.

\begin{wrapfigure}{R}{0.45\textwidth}
	\vspace{-0.8cm}
	\centering
	\includegraphics[width=0.44\textwidth]{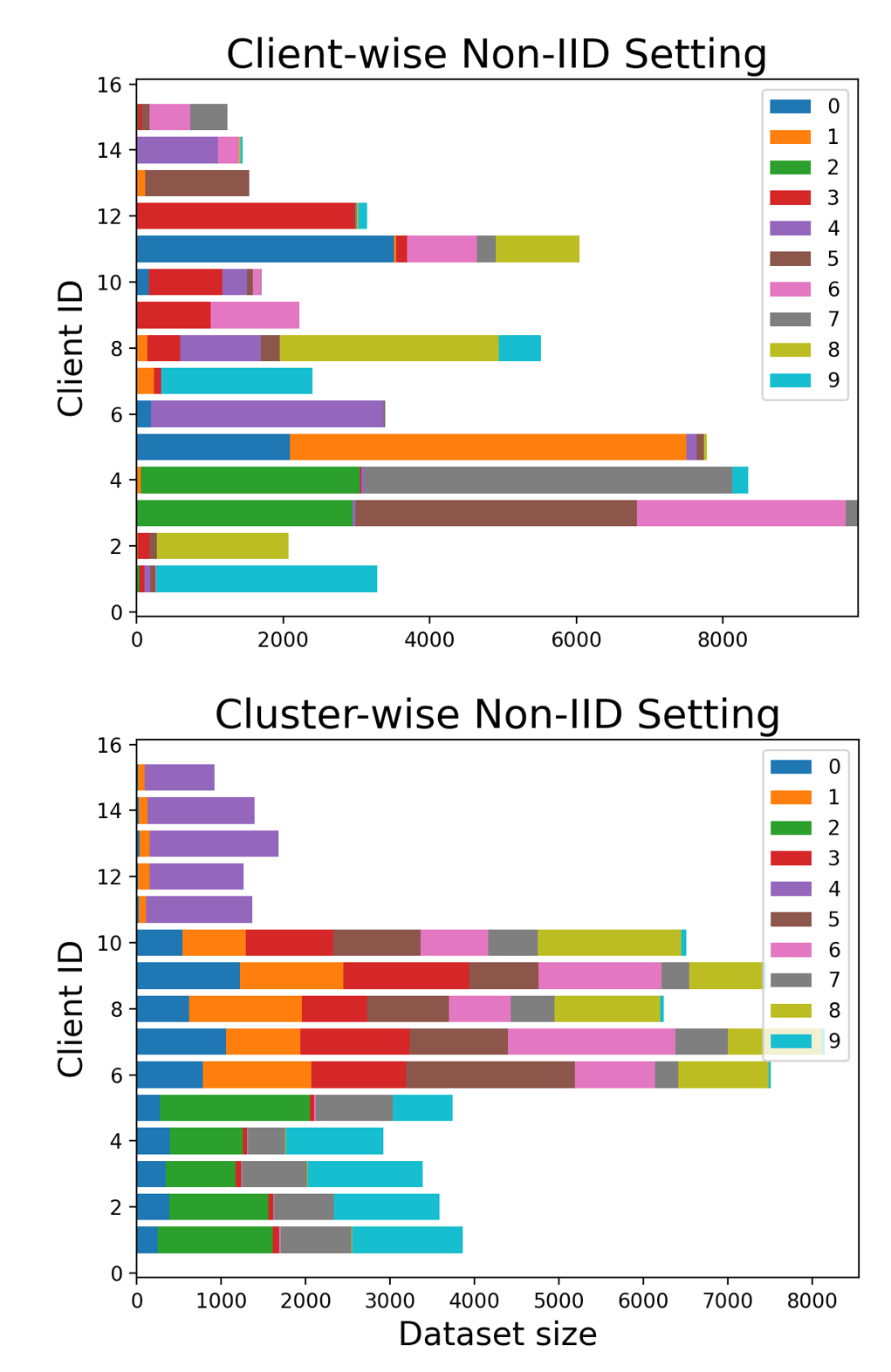}
	\caption{\small Examples of client-wise and cluster-wise non-IID. Color labels represent 10 classes and the length of bar represents the \# of instances. 
	}
	\vspace{-0.5cm}
	\label{client_cluster}
\end{wrapfigure}

Most of existing personalized FL research focus on client-wise non-IID setting that does not assume any complicated structure. For example, using Dirichlet distribution with hyperparameter $\alpha$ to simulate the non-IID data generation or partition across clients \cite{hsu2019measuring}. However, cluster-wise data are more common scenarios in real applications, such as segmenting users by demographic features including gender, age, location, etc. Moreover, there is a general assumption that clients with similar backgrounds are very likely to make similar decisions, thus generating the data with similar distributions. In the meantime, users with various backgrounds are very likely to have very different actions when encountering the same scenarios. This assumption is widely applied to population-based marketing strategy and cohort-based user behavior analytics. 

As introduced by \cite{kairouz2021advances}, the non-IID could be categorized to feature distribution skew, label distribution skew, concept drift, and quantity skew. In this paper, we will further expand the non-IID taxonomy in another dimension to be client-wise non-IID, and cluster-wise non-IID. As illustrated in \textbf{Figure \ref{client_cluster}}, the client-wise non-IID has a large variance in label distributions across clients, and then the cluster-wise non-IID has a large variance across inter-cluster clients while having a very small variance across intra-cluster clients. In general, the clustered FL will perform the best on cluster-wise non-IID data. In client-wise non-IID, the clustered FL method can outperform single model-based FL by leveraging multiple cluster-wise personalized models to alleviate the non-IID issue. In the meantime, the clustered FL is also a competitive solution that can properly balance the model personalization and generalization, while the client-wise personalized FL is usually sensitive to the over-fitting in local fine-tuning.  

%

There are various existing clustered FL methods \cite{ghosh2020efficient,mansour2020three,xie2021multi,sattler2020clustered}.
However, clusterability of clients is not well studied in the existing clustered FL methods that usually treats clustering as an add-on component for FedAvg framework \cite{mcmahan2017communication}. Moreover, a few fundamental problems still need to be further studied, such as how to represent a client and measure distance in a clustering procedure in the FL contexts, how to measure the clusterability and clustering quality that should be integrated with the learning objective of FL system.

This paper takes the first step towards addressing the above problems. We start by revisiting the existing clustered FL and formulate them into a unified bi-level optimization problem. We then propose a \textbf{We}ighted \textbf{C}lustered \textbf{F}ederated \textbf{L}earning (WeCFL) framework that represents each client by their model parameters and measures their distance by Euclidean distance in parameter space. Moreover, WeCFL keeps consistent with the weighted loss in FL by considering weighted clients for clustering. All these components are combined into a learning process in cluster-wise non-IID federated setting, in which we study the clusterability among FL clients. We then develop a new theoretical analysis framework to conduct convergence analysis on FL with non-IID data. 

Our contributions are summarized as below.
\begin{itemize}
	\item We propose the first cluster-wise non-IID setting in FL. 
	\item We formulate the clustered FL problem into a unified bi-level optimization framework. 
	\item We propose a novel Weighted Clustered Federated Learning (WeCFL) algorithm.
	\item We propose a new theoretical framework for conducting convergence analysis in clustered FL by considering a new clusterability measure $B$ in a widely-used framework.
\end{itemize}

The remaining sections of the paper are organized as follows. Section \ref{sec:related-work} introduces related work. We will discuss the clustered FL through a new perspective in Section \ref{sec:perspective}, and then formulate the problem in Section \ref{sec:problem}. The methodology is introduced in Section \ref{sec:methodology} with convergence analysis in Section \ref{sec:convergence}. Experimental settings and empirical study are discussed in Section \ref{sec:exp-setting} and \ref{sec:exp-analysis}, respectively.


\section{Related Work} \label{sec:related-work}

\subsection{Federated learning with non-IID}
The vanilla FL method, FedAvg \cite{mcmahan2017communication}, has been suffering from the non-IID challenge where each client's local data distribution is varied \cite{kairouz2021advances}. To tackle this challenge,
\cite{li2019feddane} proposed FedDANE by adapting the DANE to a federated setting. In particular, FedDANE is a federated Newton-type optimization method. \cite{li2020federated} proposed FedProx for the generalization and re-parameterization of FedAvg. It adds a proximal term to clients' local objective functions by constraining the parameter-based distance between the global model and local model. \cite{reddi2020adaptive} proposes to use adaptive learning rates to FL clients and \cite{jiang2020decentralized} conduct attention-based adaptive weighting to aggregate clients' models. \cite{li2019convergence} studies the convergence of the FedAvg on non-IID scenarios.

\textbf{Cluster-wise PFL}, also named Clustered FL, is to partition users into several groups and then train cluster-wise personalized models correspondingly. Kmeans-based clustered FL \cite{xie2021multi} measured the distance using model parameters and accuracy respectively. Hierarchical clustering \cite{briggs2020federated} has been applied to FL either. CFL \cite{sattler2020clustered} divides clients into two partitions based on the cosine similarity of the client gradients, then checks whether a partition is congruent by the norm of client gradients (hierarchy clustering). \cite{ghosh2020efficient} studied the framework of both one-shot and iterative clustered FL (IFCA). The cluster index of minimum loss for all clients was used for the assignment, which was also studied by HypCluster \cite{mansour2020three}. Few-shot clustering has been introduced to clustered FL by \cite{Dennis2021-wc,awasthi2012improved}.  FedP2P \cite{chou2021efficient} let clents in one cluster to communicate with each other to be communication-efficient.

\textbf{Client-wise PFL} usually assumes each client's data distribution is different from others; thus, each client should have a personalized model on its device. 
A naive PFL method is to learn a global model at server while conducting local fine-tuning on each client \cite{cheng2021fine,fallah2020personalized}. Ditto \cite{li2021ditto} was proposed as a bi-level optimization framework for PFL while considering a regularization term to constrain the distance between the local model and global model. The Model-Agnostic Meta-Learning (MAML) framework is also studied to personalize the clients \cite{fallah2020personalized}. Research \cite{t2020personalized} uses Moreau envelopes as clients’ regularized loss functions to optimize a bi-level problem for PFL. FedRep \cite{collins2021exploiting} learns a globally shared representation and a locally personalized head for each client. Investigations by \cite{shamsian2021personalized,chen2018federated} that aim to train a global hyper-network or meta-learner instead of a global model before sending it to clients for local optimization. SCAFFOLD \cite{karimireddy2020scaffold} proposes to learn personalized control variates that correct the local model accordingly. Layer-wise personalization \cite{arivazhagan2019federated,liang2020think} and Representation-wise personalization \cite{tan2021fedproto} are two simple but effective solution of PFL. Hermes \cite{li2021hermes} and LotterFL \cite{li2021lotteryfl} are two PFL methods considering communication efficiency for mobile clients. 

\subsection{Convergence analysis of FL}
There are few works about the convergence analysis of Clustered FL on non-IID data, but many works about FL on non-IID data. Such works can be traced back to the convergence analysis of Local SGD \cite{stich2018local,khaled2020tighter}, which differs from FedAvg in local update epochs and some special settings such as non-IID, straggler and privacy attack. Since almost all of the algorithms in FL are solved by stochastic gradient descent (SGD), the convergence analysis is usually based on the SGD convergence analysis framework. In work by \cite{li2019convergence}, the convergence of FedAvg on non-IID and partial participation is analyzed in detail, and the convergence rate is $O(\frac{1}{T})$. The impacts of some hyperparameters, such as local epochs, are also discussed. A guide by \cite{wang2021field}
provides recommendations and guidelines on how to formulate, design, evaluate and analyze FL optimization algorithms, in which convergence analysis is discussed in a separate section. 
There are some recent works \cite{xingbig, li2021ditto} model client-wise PFL task into a bi-level optimization framework and then conduct convergence analysis.


\section{A New Perspective for Clustered FL} \label{sec:perspective}

Existing clustered FL methods focus on the learning process in a federated setting, thus, the clustering components are an add-on part of the overall learning process in the FL system. We will rethink the clustered FL from a clustering perspective while considering the FL contexts. To conduct clustering in the FL system, there are several major challenges that need to be resolved.

\begin{itemize}
	\item Challenge 1: How to represent an FL client into an instance or point in clustering?
	\item Challenge 2: How do measure the distance or similarity for FL clients?
	\item Challenge 3: How to evaluate the quality of clustering by considering the FL's objective?
	\item Challenge 4: How to choose a clustering algorithm to be integrated with the FL?
\end{itemize}

For Challenge 1, existing Clustered FL methods usually use \textbf{client-specific models} to represent the client in a clustering. Using model parameters will be a straightforward solution that is to be consistent with the setting of FL. An alternative option is to use technology, e.g. federated generative adversarial learning \cite{rasouli2020fedgan} and federated representation learning \cite{zhang2020federated,li2021model}, to transform the client-specific dataset or distribution into a vector to represent the client. However, the operation of embedding datasets usually cause extra privacy concern for end-users, thus it still be a controversial topic in practice. 

For Challenge 2, the selection of distance and similarity metrics is highly reliant on the selection of client-specific representation - the solution of Challenge 1. With the given representation vector, some clustered FL reuse the classical distance and similarity measurement, such as Euclidean distance \cite{xie2021multi}, cosine similarity \cite{sattler2020clustered} and KL divergence \cite{lee2021robust}. Moreover, a key issue for this challenge is to ensure the \textbf{clusterablity} for the clients or with the given representation space and distance metric.

For Challenge 3, a basic rule of evaluation is that a "good" clustering result should also lead to a "good" learning result of the FL system. The widely used objective function of FL is a weighted sum loss of all clients, e.g. FedAvg\cite{mcmahan2017communication}. Therefore, the \textbf{client-specific weights} are important indicators to design clustering evaluation criteria in the FL context.

For Challenge 4, selecting clustering algorithms depends on the design of client-specific representation, distance metrics and evaluation criteria. Due to the complexity of the FL system requiring efficient communication and computation, a \textbf{simple clustering algorithms} is a preferred choice, such as K-means \cite{xie2021multi} or hierarchical clustering \cite{briggs2020federated}.


\section{Problem Formulation}  \label{sec:problem}
An FL system is usually composed of $m$ clients where each client needs to train an intelligent task using its own dataset $D_i$. We list the FL related notations in the first part of Table~\ref{notation} while the clustering related components are introduced in the second part.

\begin{table}[tphb!]
	\vspace{-0.2cm}
	\caption{Table of partial notations}
	\label{notation}
	\vspace{-0.1cm}
	\begin{center}
		{\small
			\resizebox{1\columnwidth}{!}{%
				\begin{tabular}{p{0.15\linewidth} | p{0.15\linewidth} | p{0.7\linewidth}}
					\toprule
					Components & Notation & Definition \\
					\midrule
					\multirow{6}{*}{FL} & $m$ & Number of clients in FL system \\
					& $D_i, |D_i|$ & The dataset and its size on Client $i$  \\ 
					& $h(\omega_i), h_i$ & Hypothesis of Client $i$ with parameter $\omega_i$  \\
					& $l(h_i,D_i), l_i$ & Loss function of Client $i$\\
					& $\eta_i^{(t)}$ & The learning rate for Client $i$ in Iteration $t$\\
					& $Q$ & Number of local update steps \\
					\toprule
					\multirow{12}{*}{Clustering} & $K$ & Number of clusters \\
					& $r_{i,k} \in {\mathbb R}^{m*K}$ & The assignment matrix, $r_{i,k}=1$ if $i\in k$ else $r_{i,k}=0$ \\ 
					& $i\in k$ & Client $i$ belongs to Cluster $k$  \\
					& $g_i$ & General form to represent Client $i$ depending on $h_i$, $l_i$, $D_i$ or something else, e.g. model parameters or loss \\
					& $G_k$ & General form to represent the centroid of Cluster $k$, and usually a linear combination of $g_i$ with $ i \in k$ \\
					& $d(g_i,G_k)$ & The distance function of general representations between Client $i$ and the center of Cluster $k$, e.g. Euclidean distance.\\
					& $H(\Omega_k), H_k$ & Hypothesis of Cluster $k$  \\
					& $\mathcal{L}(H_k), \mathcal{L}_k$ & Loss function of Cluster $k$\\
					& $\lambda_{i}$ & The importance weight of Client $i$ in Cluster $k$, and $\sum_{i\in k} \lambda_{i} =1$ \\
					\bottomrule
				\end{tabular}
			}
		}
	\end{center}
	\vspace{-0.3cm}
\end{table}

In particular, we can reformulate HypCluster \cite{mansour2020three} and IFCA \cite{ghosh2020efficient} as a bi-level optimization problem:
\begin{subequations}
	\label{eq:ifca}
	\begin{align} 
		\operatorname*{minimize}_{\{h_k\}}\ &\frac{1}{m} \sum_{k=1}^{K}\sum_{i=1}^{m}r_{i,k}\mathcal{L}(H_k, D_i) \label{obj:ifca} \\
		\textit{subject to }
		& r_{i,k} = \operatorname*{argmin}_{r_{i,k}} \mathcal{L}(H_k, D_i) 
	\end{align} 
\end{subequations}

We also formulate the FeSEM \cite{xie2021multi} in a bi-level optimization framework.
\begin{subequations}
	\label{eq:fesem}
	\begin{align}
		\operatorname*{minimize}_{\{\Omega_k\}}\ &\frac{1}{m} \sum_{k=1}^{K}\sum_{i=1}^{m}r_{i,k} \mathcal{L}(\Omega_k, D_i) \\
		\textit{subject to }
		& {r_{i,k}} = \operatorname*{argmin}_{r_{i,k}} \frac{1}{m}\sum_{k=1}^{K}\sum_{i=1}^{m}r_{i,k} \| \omega_i - \Omega_k \|_2^2
	\end{align}
\end{subequations}
where $\Omega_k = \frac{1}{\sum_{i \in k} r_{i,k}} \sum_{i \in k} \omega_i$ is the centroid of the cluster $k$.


\section{Methodology} \label{sec:methodology}

As we mentioned in Section \ref{sec:perspective}, the client-wise importance weights are important indicators for clustering to be consistent with the loss function in FL. Therefore, we design a general form of the objective function for clustered FL problem by considering weighted clustering, which is a bi-level optimization problem. The previous works could be special cases of our proposed form by weighing clients equally.
\begin{subequations}
	\label{eq:method}
	\begin{align}
		&\operatorname*{minimize}_{\{\Omega_k\}}\ \mathcal{R}=\frac{1}{\sum_{j=1}^m \lambda_j} \sum_{k=1}^{K}\sum_{i=1}^{m}r_{i,k}\lambda_{i} \mathcal{L}_k(D_i)  \label{upper} \\
		&\textit{subject to }
		{r_{i,k}}= \operatorname*{argmin}_{r_{i,k}} \mathcal{F}:\frac{1}{\sum_{j=1}^m \lambda_j}\sum_{k=1}^{K}\sum_{i=1}^{m}r_{i,k}\lambda_{i} d(g_i,G_k) \label{low}
	\end{align}
\end{subequations}
where $\lambda_i$ is the importance weight for the client $i$ in the cluster $k$.

\begin{wrapfigure}{r}{0.5\textwidth}
	\vspace{-0.3cm}
	\begin{minipage}{0.5\textwidth}
		\begin{algorithm}[H]
			\caption{\small Weighted Clustered FL (WeCFL)}
			\label{alg}
			\footnotesize
			\begin{algorithmic}
				\STATE {\bfseries Input:} $K,\{D_1,D_2,\dots,D_m\},\{l_1,l_2,\dots,l_m\}$
				\STATE {\bfseries Initialize:} Randomly select $\{H_1,H_2,\dots,H_K\}$
				\REPEAT
				\STATE {\bfseries Expectation step:} Assign Client $i$ to Cluster $k$ by $$k=\operatorname*{argmin}_k \lambda_{i} d(g_i,G_k)$$
				\STATE {\bfseries Maximization (Aggregation) step:} Compute cluster center $H_k$ by minimize $$\mathcal{F}=\frac{1}{\sum_{j=1}^m \lambda_j}\sum_{k=1}^{K}\sum_{i=1}^{m}r_{i,k}\lambda_{i} d(g_i,G_k)$$
				\STATE {\bfseries Distribution step:} Send $H_k$ to clients in Cluster k
				\STATE {\bfseries Local update step:} Run Gradient Descent $Q$ steps using local data $D_i$ to minimize $$\mathcal{R}=\frac{1}{\sum_{j=1}^m \lambda_j} \sum_{k=1}^{K}\sum_{i=1}^{m}r_{i,k}\lambda_{i} \mathcal{L}(H_k,D_i)$$
				\UNTIL{convergence condition satisfied}
				\STATE {\bfseries Output:} $r_{i,k}, \{H_1,H_2,\dots,H_K\}$
			\end{algorithmic}
		\end{algorithm}
	\end{minipage}
	\vspace{-0.5cm}
\end{wrapfigure}

The upper-level objective \ref{upper} is an FL problem that is usually optimized by the FedAvg algorithm, whereas the lower-level objective \ref{low} is a clustering problem that is usually optimized by the EM algorithm \cite{dempster1977maximum}. It is a straightforward solution to combine these two algorithms into one and then iteratively solve the objective. 

Algorithm 1 illustrates the procedure of WeCFL to solve the proposed bi-level optimization problem in Eq. \ref{eq:method} by four main steps in every iteration. The first twp steps correspond to an EM algorithm solving the clustering problem: the E-step assigns clients to the nearest cluster and the M-step calculates the centroid of each cluster, which is equivalent to the model aggregation step of FedAvg \cite{mcmahan2017communication}. Unlike normal clustering, here the representation of each client keeps being updated by the following two steps: the server broadcasts the aggregated model for each cluster to its clients; once received the cluster model, each client applies local updates to it by minimizing the loss for its local data $D_i$ and the resulted local model is the client's new representation for the next iteration. 


\section{Convergence Analysis} \label{sec:convergence}
For the convergence of optimization problem \ref{eq:ifca}, which is used by HypCluster \cite{mansour2020three} and IFCA \cite{ghosh2020efficient}, the convergence is easy to analyze. We separate the algorithm into two steps: the assignment step, and the local update step. In the assignment step, it is always best to assign the least loss function to the clients, so the Objective \ref{obj:ifca} will not increase. In the local update step, which uses gradient descent algorithm, by choosing the proper learning rate under Assumption \ref{as2}, the Objective \ref{obj:ifca} will not increase either. Moreover, the Objective \ref{obj:ifca} will monotonously decrease, proving convergence.

For the convergence of Optimization Problem in Eq. \ref{eq:fesem} and \ref{eq:method}, we consider a special case of Problem in Eq. \ref{eq:method} that also covers Problem in Eq. \ref{eq:fesem}, in which the client representation $g$ is the parameter of the hypothesis of Client $i$, and the distance function is Euclidean norm square $\| \cdot \|_2^2$. then the objective function to minimize is as follows:
\begin{equation}
	\label{eq:wecfl}
	\begin{split}
		\operatorname*{minimize}_{\{\Omega_k\}}\ &\mathcal{R}=\frac{1}{\sum_{j=1}^m \lambda_j} \sum_{k=1}^{K}\sum_{i=1}^{m}r_{i,k}\lambda_{i} \mathcal{L}(\Omega_k, D_i) \\
		\textit{subject to }
		& {r_{i,k},\{\Omega_k\}}=		\operatorname*{argmin}_{r_{i,k},\{\Omega_k\}} \mathcal{F}:\frac{1}{\sum_{j=1}^m \lambda_j}\sum_{k=1}^{K}\sum_{i=1}^{m}r_{i,k}\lambda_{i} \| \omega_i - \Omega_k \|_2^2
	\end{split}
\end{equation}


\subsection{Convergence Analysis of $\mathcal{F}$}
To analyze the convergence of the optimization problem \ref{eq:wecfl} above, both $\mathcal{F}$ and $\mathcal{R}$ should be considered. We will first analyze the clustering objective $\mathcal{F}$:

\begin{assumption}
	\label{as_bound}
	{\rm(Unbiased gradient estimator and Bounded gradients).} The expectation of stochastic gradient $\nabla l(\omega_i, \xi)$ is an unbiased estimator of the local gradient for each client:
	\begin{align*}
		{\mathbb E}_{\xi_i \sim D_i}[\nabla l(\omega_i, \xi)] = \nabla l(\omega_i)
	\end{align*}
	and expectation of L2 norm of $\nabla l(\omega_i, \xi)$ is bounded by a constant U:
	\begin{align*}
		{\mathbb E}_{\xi_i \sim D_i}[\|\nabla l(\omega_i, \xi)\|_2] \leq U
	\end{align*}
	It is also applied for $\mathcal{L}$.
\end{assumption}

\begin{theorem}
	{\rm(Convergence of clustering problem $\mathcal{F}$).} Under Assumption \ref{as_bound}, for arbitrary communication round $t$, if $\eta_i^{(t)} \leq \frac{\|  \omega_i^{(t)}-\Omega_k\|}{Q U}$, $\mathcal{F}$ converges.
	\label{theo1}
\end{theorem}

\begin{remark}
	{\rm(Clustering stability guarantee).} It is important to make sure $\mathcal{F}$ converges, which means the clustering results to be stable. We also conduct detailed experimental analysis on clustering in Section \ref{sec:clu}.
\end{remark}

\subsection{Convergence Analysis of $\mathcal{R}$}
\begin{definition}
	\label{def:clu}
	{\rm(Clusterability measure).} For arbitrary Client i in Cluster $k$, if its gradient obeys:
	\begin{equation}\label{eq:clusterability}
		\begin{split}
			\frac{ \|\sum_{p \in k} \frac{ \lambda_{p}\nabla l(\omega_p,D_p)}{\sum_{z \in k} \lambda_{z}}  - \nabla l(\omega_i,D_i) \|_2 }{ \| \sum_{p \in k} \frac{ \lambda_{p}\nabla l(\omega_p,D_p)}{\sum_{z \in k} \lambda_{z}}\|_2}
			\leq B
		\end{split}
	\end{equation}
	
	We define the clusterability of Cluster $k$ to be $B$. If $B=0$, it means the same data distribution among clients. The larger $B$, the less clusterability of Cluster $k$. It will even lead to divergence if $B$ is too large. The experimental study is also conducted for $B$ in 
	Section \ref{sec:clu}. $B$ is very small and close to zero.
\end{definition}

\begin{assumption}
	\label{as1}
	{\rm(Convex).} Each loss function $l$ or $\mathcal{L}$ is convex. Then we will have
	\begin{equation}\label{eq:convex}
		l(y)\geq l(x) + \langle \nabla l(x), y-x \rangle 
	\end{equation}
\end{assumption}

\begin{assumption}
	\label{as2}
	{\rm(Lipschitz Smooth).} Each loss function $l$ or $\mathcal{L}$ is $\beta$-smooth.  Then we will have
	\begin{equation}\label{eq:smooth}
		l(y)\leq l(x) + \langle \nabla l(x), y-x \rangle + \frac{\beta}{2} \|y-x\|_2^2 
	\end{equation}
\end{assumption}

\begin{assumption}
	\label{as3}
	{\rm(Bounded gradient variance).} The variance of stochastic gradient $\nabla l(\omega_i, \xi)$ is bounded by $\sigma^2$,
	\begin{equation}
		\begin{split}
			&{\mathbb E}_{\xi_i \sim D_i}[\|\nabla l(\omega_i, \xi)-\nabla l(\omega_i)\|^2_2] \\  = &{\mathbb E}[\|\nabla l(\omega_i, \xi)\|^2_2] - \|\nabla l(\omega_i)\|^2_2 
			\leq \sigma^2
		\end{split}
	\end{equation}
	It is also applied for $\mathcal{L}$.
\end{assumption}

\begin{theorem} {\rm(Convergence of WeCFL).} Let Assumption \ref{as_bound}, \ref{as1}, \ref{as2} and \ref{as3} hold, when $\eta_{(t,q)} < min\{\frac{\| \omega_i^{(t)}-\Omega_k\|_2}{Q U},\frac{{\mathbb E}[\| \nabla \mathcal{L}(\Omega_k^{(t,M,q)})\|_2^2] -B U^2}{{\mathbb E}[\| \nabla \mathcal{L}(\Omega_k^{(t,M,q)})\|_2^2] +\sigma^2} \cdot \frac{2}{\beta}\}$, the EM loss function $\mathcal{F}$ converges, and the FL loss function $\mathcal{R}$ decreases monotonically, thus the WeCFL converges.
	\label{theo2}
\end{theorem}

\begin{theorem} {\rm(Convergence rate of WeCFL).} Let Assumption \ref{as_bound}, \ref{as1}, \ref{as2} and \ref{as3} hold, and $\Delta =\mathcal{R}_0-\mathcal{R}^* $, given any $\epsilon>0$, after 
	\begin{equation} \label{eq:rate}
		T \geq \frac{\Delta}{Q (\epsilon (\eta - \frac{\beta \eta^2}{2})-\frac{\beta \eta^2}{2} \sigma^2-\eta B U^2)}
	\end{equation} 
	communication rounds of WeCFL, we have 
	\begin{equation}
		\frac{1}{T Q}\sum_{k=1}^{K}\sum_{i \in k} \sum_{t=0}^{T-1} \sum_{q=0}^{Q-1} \frac{\lambda_i }{\sum_{j=1}^m \lambda_j} {\mathbb E}[\| \nabla \mathcal{L}(\Omega_k^{(t,M,q)})\|_2^2] \leq \epsilon
	\end{equation}
	\label{theo3}
\end{theorem}

\begin{remark}
	{\rm(Linear convergence rate of WeCFL).} According to Equation \ref{eq:rate}, with proper learning rate, the convergence rate of WeCFL is $O(1/T)$, which achieve the state of the art rate such as SGD and  \cite{li2019convergence}. 
\end{remark}


\section{Experimental settings}  \label{sec:exp-setting}
\subsection{Datasets} \label{sec:data}
We use two benchmark datasets as below, then conduct group-wise non-IID pre-processing on them.
\begin{itemize}
	\item \textbf{Fashion-MNIST} \cite{xiao2017fashion} consists of 70,000 28x28 grayscale images in 10 classes, with 60,000 training images and 10000 test images under the MIT License.
	\item \textbf{CIFAR-10} \cite{krizhevsky2009learning} provides 60,000 32x32 colour images in 10 classes, with 6,000 images per class under the MIT License. There are 50,000 training images and 10,000 test images. The heterogeneity of the CIFAR-10 dataset is much higher than MNIST family datasets.
\end{itemize}



The first \textbf{cluster-wise non-IID pre-processing} method is using Dirichlet distribution to control the randomness of non-IID \cite{hsu2019measuring}. Specifically, we divide the dataset into $K=10$ clusters with $\alpha=0.1$ to generate large variance on cluster-wise non-IID, and then we divide each cluster into $m/K$ clients with $\alpha=10$ to control client-wise non-IID.

The second \textbf{cluster-wise non-IID pre-processing} method is $n$-class proposed by FedAvg \cite{mcmahan2017communication} that is to assign $n$ classes out of all classes in the dataset. We randomly assign $3$ classes to each cluster with a relatively balanced number of instances per class, and then assign $2$ classes to each client.

\subsection{Baseline and system settings}
\textbf{Baseline} For single model-based FL, we choose FedAvg \cite{mcmahan2017communication} and FedProx \cite{li2020federated} with $\lambda = 0.95$ as the baselines. 
For clustered FL methods, FeSEM \cite{xie2021multi} and IFCA \cite{ghosh2020efficient} which is simlilar to HypCluster are chosen as the baselines. We also propose FedAvg+ and FedProx+ by training FedAvg and FedProx $K$ times, and then learn an ensemble model via soft voting to serve all clients.

\textbf{System settings} We generate 200 clients for simulating a relatively large-scale FL system. We use CNN \cite{lecun2015deep} as the basic model for each client. We evaluate the performance using both \textbf{micro accuracy} (\%) and \textbf{macro F1-score} on the client-wise test datasets due to high non-IID degrees. The standard deviation has been estimated for five times of experiments with different random seeds, and the mean is obtained by the last three rounds out of the total 100 communication rounds. More details of setting could be found in the Appendix.



\begin{table}[H]
	\vspace{-0.3cm}
	\caption{Performance comparison on cluster-wise non-IID }
	\label{tab:cluster}
	\begin{center}
		\resizebox{1.0\columnwidth}{!}{%
			\begin{tabular}{ll|cccc|cccc}
				\toprule
				\multicolumn{2}{c}{Datasets} & \multicolumn{4}{c}{Fashion-MNIST} & \multicolumn{4}{c}{CIFAR-10} \\
				
				\midrule
				\midrule
				
				\multicolumn{2}{c}{Non-IID setting} & \multicolumn{2}{c}{$\alpha=(0.1,10)$} & \multicolumn{2}{c}{$(3,2)-$class} & \multicolumn{2}{c}{$\alpha=(0.1,10)$} & \multicolumn{2}{c}{$(3,2)-$class}  \\
				
				\midrule
				
				\textbf{K} & Methods & Accuracy & Macro-F1 & Accuracy & Macro-F1  & Accuracy & Macro-F1 & Accuracy & Macro-F1 \\
				
				\midrule
				
				\multirow{3}{*}{\textbf{1}} & FedAvg & 86.08$\pm$0.70 & 57.24$\pm$2.26 & 86.33$\pm$0.44 & 46.09$\pm$1.08  & 24.38$\pm$3.30  & 11.69$\pm$3.15  & 21.33$\pm$3.83 & 9.0$\pm$0.58           \\
				
				& FedProx &  86.32$\pm$0.78 & 58.03$\pm$3.19 & 86.42$\pm$0.63 & 45.86$\pm$1.42   & 24.73$\pm$3.68  & 11.28$\pm$2.35 & 22.66$\pm$1.13 & 9.23$\pm$0.78     \\ 
				
				
				\midrule
				\multirow{5}{*}{\textbf{5}} & FedAvg+    &   87.61 & 59.48    & 86.95   & 65.61   &  25.97  & 12.16  &        24.35   &  9.06            \\ 
				
				& FedProx+ & 87.94 & 59.83 & 86.52 & 65.73  & 26.05 & 12.53   &   24.83 & 9.31      \\ 
				
				& IFCA & 84.60$\pm$2.22 & 62.03$\pm$3.01 & 84.94$\pm$2.54 & 66.50$\pm$4.43  & 34.1$\pm$4.79  & 22.12$\pm$2.21 & 29.80$\pm$4.49 & 17.90$\pm$2.08   \\  
				
				& FeSEM & 94.64$\pm$1.54 & 82.90$\pm$2.38 & 94.20$\pm$1.96 & 77.07$\pm$6.05 & 59.06$\pm$3.24 & \textbf{32.33$\pm$7.25} & 58.76$\pm$3.35 & 35.75$\pm$2.54 \\
				
				& WeCFL & \textbf{94.64$\pm$1.02} & \textbf{84.4$\pm$1.31} & \textbf{94.97$\pm$1.43} & \textbf{77.36$\pm$3.94} & \textbf{59.26$\pm$3.32} & 32.26$\pm$3.46 & \textbf{62.44$\pm$2.53} & \textbf{38.55$\pm$1.76} \\
				
				\midrule
				\multirow{5}{*}{\textbf{10}} & FedAvg+    &   89.42 & 67.83    & 86.91    &  63.01      &   28.45  & 13.79  &        27.28  &  9.81           \\ 
				
				& FedProx+ & 89.55 & 68.02 &  86.73  & 63.42   & 28.33  & 13.64   &   26.94 & 9.64      \\ 
				
				& IFCA  &  82.10$\pm$5.40  & 62.62$\pm$8.22   &86.58$\pm$4.97 & 66.22$\pm$5.69    &   34.84$\pm$5.82 &22.76$\pm$3.99    &  34.06$\pm$2.60  & 18.7$\pm$1.31 \\ 
				
				& FeSEM  &  95.73$\pm$1.28  & 89.34$\pm$1.57   & 95.54$\pm$0.74 & 84.43$\pm$2.38  &  66.89$\pm$2.18 & 38.35$\pm$4.24  &    71.76$\pm$2.23 & 49.72$\pm$3.84 \\ 
				
				& WeCFL &  \textbf{95.88$\pm$0.85}   & \textbf{89.81$\pm$1.59}  & \textbf{97.10$\pm$0.51} & \textbf{88.96$\pm$1.36}  & \textbf{70.95$\pm$3.57} & \textbf{40.19$\pm$2.88} & \textbf{72.13$\pm$1.88} & \textbf{50.65$\pm$2.15}\\ 
				\bottomrule
			\end{tabular}
		}
	\end{center}
	\vspace{-0.5cm}
\end{table}

\section{Experimental analysis} \label{sec:exp-analysis}
\subsection{Comparison study}
\textbf{Table \ref{tab:cluster}} shows performance comparison on cluster-wise non-IID setting. Measured by client-wise test dataset-based micro accuracy and macro F1-score, WeCFL outperforms almost all baselines on Fashion-MNIST and CIFAR-10 datasets. IFCA also doesn't show a very competitive performance on both two datasets. One of the main reasons is due to IFCA's unstable clustering capability. IFCA's clustering procedure is not a usual clustering algorithm with a well-defined distance or similarity metric. Specifically, in IFCA's clustering procedure, the similarity metric is based on how the cluster-specific model performs on the client's local dataset. This kind of metric is unlike other classic distance and similarity metrics which have demonstrated good characteristics from geometry and algebra perspectives. 

Within a proper interval, larger $K$ leads to better performance. As shown in the figure, when K is increased from 5 to 10, all methods' performance is increased. However, IFCA sometimes decreases its performance due to its unstable clustering capability. The FedAvg and FedProx perform very badly on CIFAR-10 that demonstrating their inability to tackle group-wise non-IID data. Their ensemble extension, FedAvg+ and FedProx+, can slightly increase the performance because the model's generalization has been improved by leveraging ensemble learning. It is noteworthy that FedAvg+ and FedProx+ are very stable by assembling multiple models; thus we didn't measure the variance of these ensemble models.

\subsection{Convergence analysis}

\textbf{Figure \ref{fig:convergence-cfl}} shows the convergence curves of three clustered FL methods including IFCA, FeSEM and WeCFL. Two figures measure performances on test accuracy and macro F1, respectively. The experimental dataset is derived from CIFAR-10 by preprocessing the dataset with a cluster-wise non-IID setting. Specifically, the non-IID of (3,2)-class that assigns three classes to each cluster while assigning two classes to each client. As shown in the figures, WeCFL converges faster than others.

\begin{figure}[H]
	\vspace{-0.4cm}
	\centering
	\begin{minipage}{0.48\textwidth}
		\centering
		\includegraphics[width=1\textwidth]{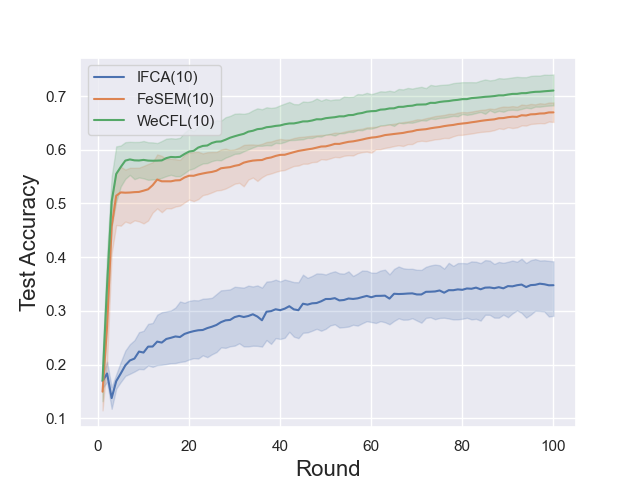} 
		\vspace{-0.2cm}
		\label{fig:acc}
	\end{minipage}\hfill
	\begin{minipage}{0.48\textwidth}
		\centering
		\includegraphics[width=1\textwidth]{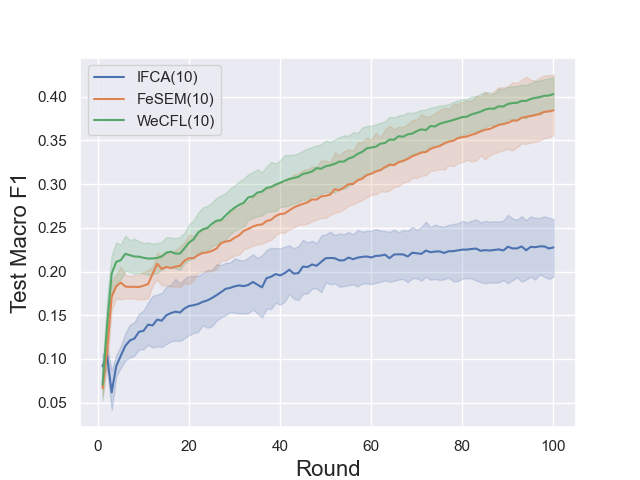} 
		\vspace{-0.3cm}
		\label{fig:f1}
	\end{minipage}
	\caption{\small Convergence of \textbf{clustered FL} methods on \textbf{CIFAR-10} under the \textbf{(3,2)-class} non-IID setting}
	\vspace{-0.5cm}
	\label{fig:convergence-cfl}
\end{figure}

\begin{figure}[H]
	\vspace{-0.4cm}
	\centering
	\begin{minipage}{0.48\textwidth}
		\centering
		\includegraphics[width=1\textwidth]{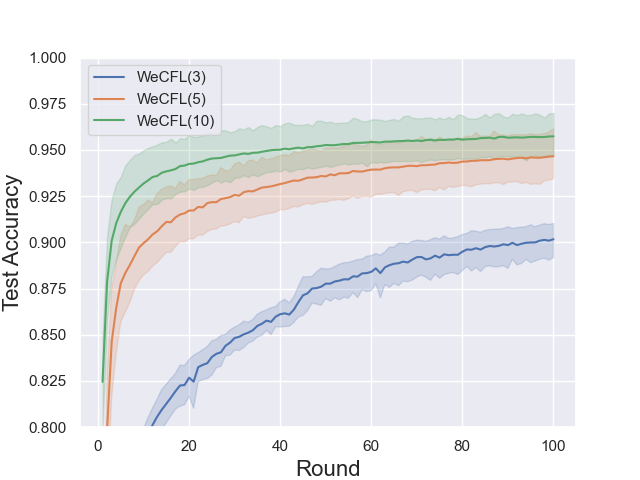} 
		\vspace{-0.5cm}
	\end{minipage}\hfill
	\begin{minipage}{0.48\textwidth}
		\centering
		\includegraphics[width=1\textwidth]{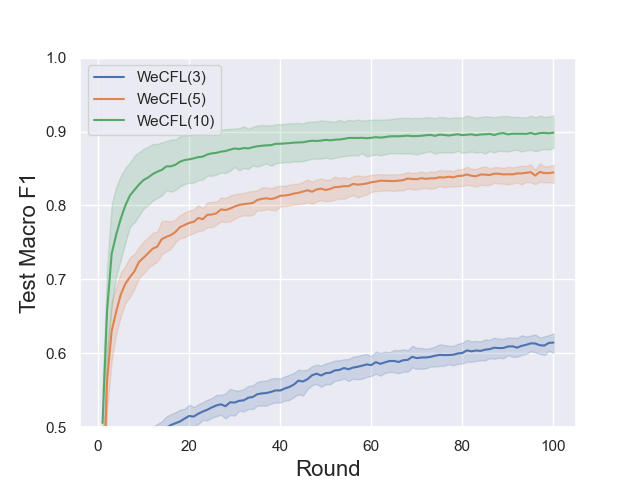} 
		\vspace{-0.3cm}
		\label{fig:f1_k}
	\end{minipage}
	\caption{\small Convergence of \textbf{WeCFL} on \textbf{Fashion-MNIST} under the $\boldsymbol{\alpha=(0.1,10)}$ non-IID setting}
	\vspace{-0.5cm}
	\label{fig:convergence-wecfl}
\end{figure}


Figure \ref{fig:convergence-wecfl} demonstrates that WeCFL can convergence in different K. The experimental dataset is derived from Fashion-MNIST using the Dirichlet-based group-wise non-IID pre-processing method with $\alpha={0.1,10}$. Specifically, we use a Dirichlet distribution with $\alpha=0.1$ to control the inter-cluster non-IID with large variance, and then use another Dirichlet distribution with $\alpha=10$ to control intra-cluster client-wise non-IID with small variance. The figures demonstrate that a larger $K$ is more likely to lead to better performance on both test accuracy and macro F1 score.

\subsection{Clustering study}\label{sec:clu}

\textbf{Clustering evaluation} A good clustering generally satisfies two evaluation criteria: the clients in the cluster are similar to each other, and the clusters are dissimilar to each other. We use cosine similarity to measure the difference among clients or clusters generated by WeCFL. Figure \ref{similarity} visualizes the inter-cluster and intra-cluster similarities. Specifically, the left figure shows the similarity among 10 clusters' centroids, and its similarity value is around 0.93 that indicating a big difference among the clusters. The right figure is the similarity among 20 intra-cluster clients that all of them are bigger than 0.999. In summary, Figure \ref{similarity} demonstrates that WeCFL can distinct clusters (left figure) and group similar clients into the same cluster (right figure).


\begin{figure}[h]
	\vspace{-0.2cm}
	\begin{center}
		\centerline{\includegraphics[width=0.8\columnwidth]{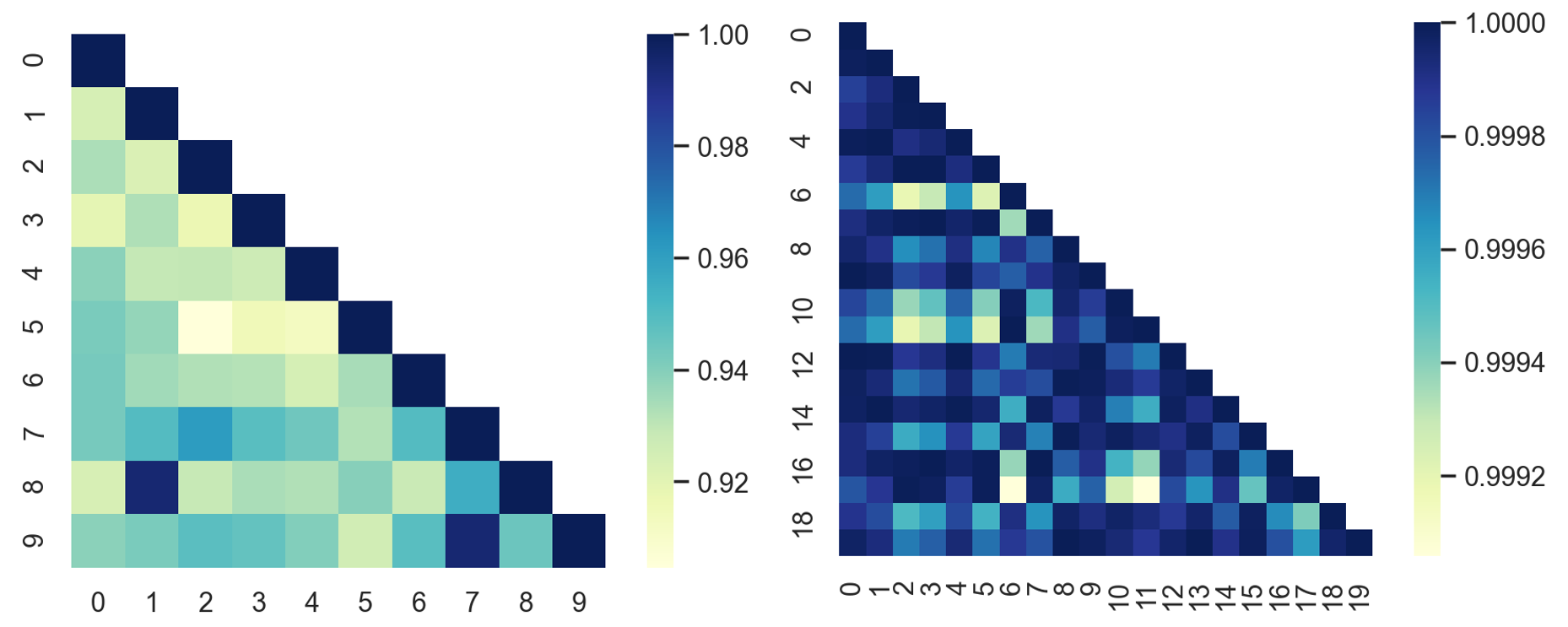}}
		\caption{Cosine similarity heatmap of 10 clusters' centroids (left) and 20 clients in a cluster (right).}
		\label{similarity}
	\end{center}
	\vspace{-0.5cm}
\end{figure}

\textbf{Clustering visualization}
To verify the effectiveness of the proposed WeCFL method and whether the clients are clustered properly, we visualize the clustering results using t-SNE \cite{tang2016visualizing} to transform client-wise representations into two-dimensional vectors. All clustering results are generated by WeCFL. As shown in Figure ~\ref{fig:tsne}, it is obvious clusters are distinguishable from each other, which indicates that the clustering results are learned perfectly. The highly-dense clusters of markers also indicate that $B$ in Eq. \ref{eq:clusterability} is very small, which can also be verified in values. It is also worth noting that the clustering algorithm converges very fast. In general, it takes no more than 10 communication rounds to achieve convergence on clustering (more details in Appendix). Once clustering converges, the operations on later communication rounds are equivalent to conducting a cluster-specific FedAvg.


\begin{figure}[h]
	\vspace{-0.4cm}
	\centering
	\begin{minipage}{0.48\textwidth}
		\centering
		\includegraphics[width=1\textwidth]{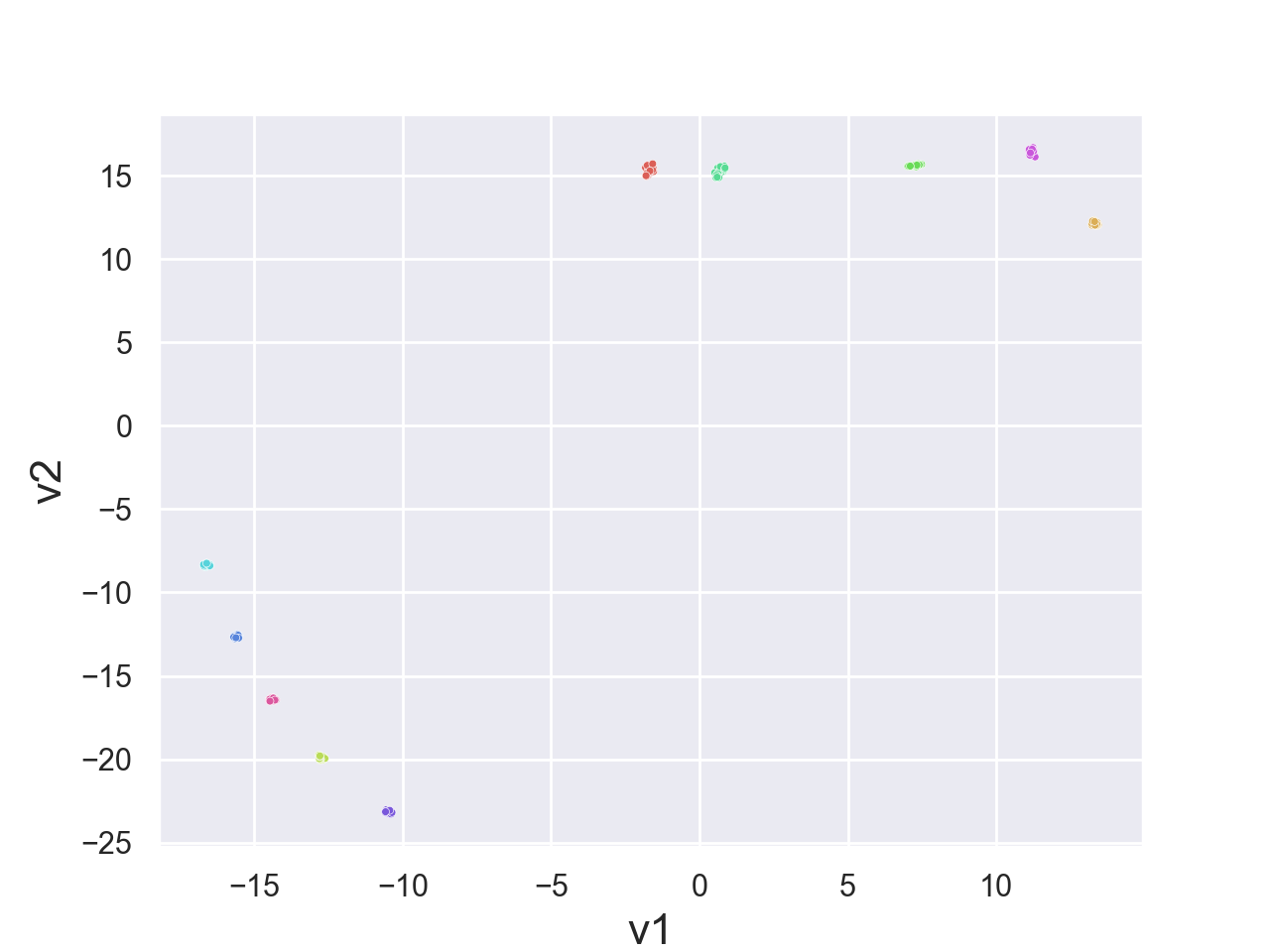} 
		\vspace{-0.5cm}
	\end{minipage}\hfill
	\begin{minipage}{0.48\textwidth}
		\centering
		\includegraphics[width=1\textwidth]{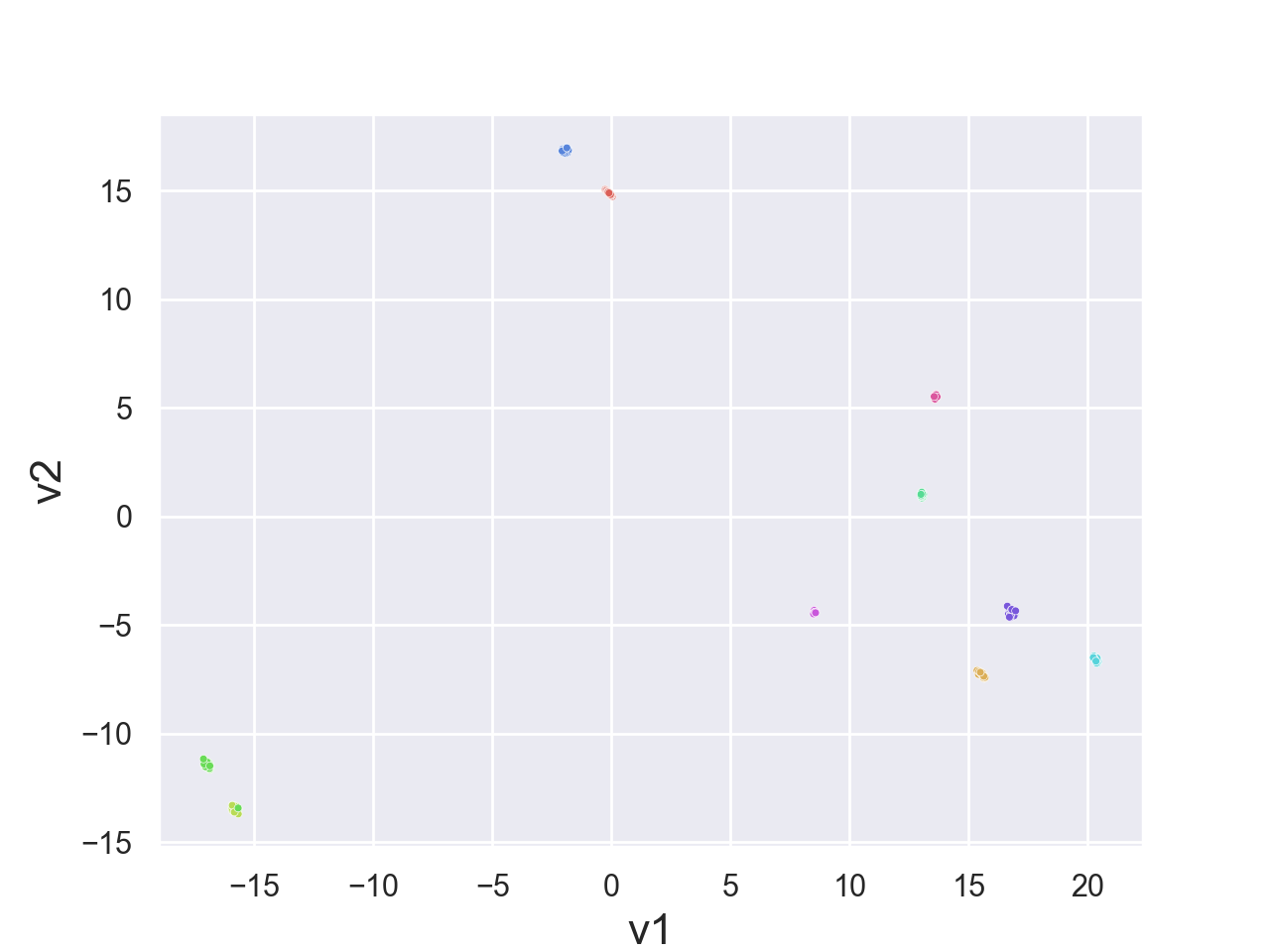} 
		\vspace{-0.3cm}
	\end{minipage}
	\caption{\small T-SNE visualization \cite{tang2016visualizing} of clustering results on the Fashion-MNIST (left) and CIFAR-10 (right) under the $\alpha=(0.1,10)$ cluster-wise non-IID setting, generated by 200 clients across ten clusters after only five communication rounds before server aggregation. Different colors represent different cluster labels.}
	\vspace{-0.5cm}
	\label{fig:tsne}
\end{figure}


\section{Conclusion}  \label{sec:conclusion}
This work rethinks the clustered FL from a new perspective on clustering, and then proposes a general form for clustered FL. A weighted clustering has been applied to clustered FL. The most important contribution is proposing a new convergence analysis to the general form of clustered FL. Experiments on both cluster-wise non-IID settings can support our claims.


\printbibliography

\pagebreak{}

\appendix

\section{Proof of Convergence Analysis}
There are more notations to be defined in the subsequent proof. Superscripts $E, M, D, L$ represent the four steps in WeCFL. For example, $\mathcal{F}^{(t,L)}$ represents $\mathcal{F}$ in the local update step of round t. And $q$ represents the $q$-th step in local update with total local steps $Q$.
\subsection{Proof of Theorem \ref{theo1}}
\begin{lemma}\label{l1}
	In the Expecation step of communication round t+1, fix $\omega, \Omega$, and assign $r_{i,k}=1$ if
	$$
	k=\operatorname*{argmin}_k \| \omega_i-\Omega_k\|_2^2
	$$
	then we can prove that:
	\begin{equation}
		\mathcal{F}^{(t+1,E)}\leq \mathcal{F}^{(t,L)}
	\end{equation}
\end{lemma}

\begin{proof}
	$r^{(t+1)}_{i,k}=1$ is to find the right k for Client i to minimize $\| \omega_i-\Omega_k\|_2$, which means to find the shortest Euclidean distance from each $\Omega_1, \Omega_2, \dots, \Omega_K$ to $\omega_i$, so for every i,
	$$
	\lambda_{i} \| \omega_i - \Omega_k^{(t+1,E)} \|_2^2 \leq \lambda_{i} \| \omega_i - \Omega_k^{(t,L)} \|_2^2
	$$
	then sum it with from $i = 1$ to $m$, we can easily get:
	$$
	\mathcal{F}^{(t+1,E)}\leq \mathcal{F}^{(t,L)}
	$$
\end{proof}

\begin{lemma}\label{l2}
	In the Maximization step of communicaiton round t, fix $r,\omega$, define:
	\begin{equation}
		\Omega^{(t,M)}_k =\sum_{i \in k} \frac{ \lambda_{i}}{\sum_{j \in k} \lambda_{j}} \omega_i
	\end{equation}
	we can prove that:
	\begin{equation}
		\mathcal{F}^{(t,M)}\leq \mathcal{F}^{(t,E)}
\end{equation}
\end{lemma}

\begin{proof}
For an arbitrary Client i in Cluster k, the loss square is :
\begin{equation}
	\begin{split}
		\lambda_{i}\|  \omega_i-\Omega_k^{(t,E)}\|_2^2 &= \lambda_{i}\| \omega_i-\Omega^{(t,M)}_k+\Omega^{(t,M)}_k-\Omega_k^{(t,E)}\|_2^2\\
		&=\lambda_{i}\|  \omega_i-\Omega^{(t,M)}_k\|_2^2+\lambda_{i} \| \Omega^{(t,M)}_k-\Omega_k^{(t,E)}\|_2^2\\
		&\quad+2 \lambda_{i}\langle\, \omega_i-\sum_{i \in k} \frac{ \lambda_{i}}{\sum_{j \in k} \lambda_{j}} \omega_i,\sum_{i \in k} \frac{ \lambda_{i}}{\sum_{j \in k} \lambda_{j}} \omega_i-\Omega_k^{(t,E)}\rangle\\
	\end{split}
\end{equation}
then sum all the clients in Cluster k together:
\begin{equation}
	\begin{split}
		\sum_{i \in k}\lambda_{i}\|  \omega_i-\Omega_k^{(t,E)}\|_2^2 & =
		\sum_{i \in k}\lambda_{i}\|  \omega_i-\Omega^{(t,M)}_k\|_2^2+ \sum_{i \in k}\lambda_{i} \| \Omega^{(t,M)}_k-\Omega_k^{(t,E)}\|_2^2\\
		&+2\langle\,\sum_{i \in k}\lambda_{i} \omega_i-\sum_{i \in k}\lambda_{i} \sum_{i \in k} \frac{ \lambda_{i}}{\sum_{j \in k} \lambda_{j}} \omega_i,\sum_{i \in k} \frac{ \lambda_{i}}{\sum_{j \in k} \lambda_{j}} \omega_i-\Omega_k^{(t,E)}\rangle\\
		& =\sum_{i \in k}\lambda_{i}\|  \omega_i-\Omega^{(t,M)}_k\|_2^2+ \sum_{i \in k}\lambda_{i} \| \Omega^{(t,M)}_k-\Omega_k^{(t,E)}\|_2^2
	\end{split}
\end{equation}
So sum all loss functions of all clusters, we can get:
\begin{equation}
	\mathcal{F}^{(t,M)} - \mathcal{F}^{(t,E)} =- \frac{1}{\sum_{j=1}^m \lambda_j}\sum_{k=1}^{K} \sum_{i \in k}\lambda_{i} \| \Omega^{(t,M)}_k-\Omega_k^{(t,E)}\|_2^2 \leq 0
\end{equation}

\end{proof}


\begin{lemma}\label{l3}
Under Assumption \ref{as_bound}, in the Distribution step of communicaiton round t+1, we get $\omega_{i \in k}=\Omega_k$. In the Local update step of communicaiton round t+1, fix $r;\Omega$, after Q steps, define:
\begin{equation}
	\omega_i^1=\omega_i^0-\eta_i^{(t)}*\nabla l_i(\omega_i^0,D_i),\dots
\end{equation}
So
\begin{equation}\label{eq:wi}
	\begin{split}
		\omega_i^{(n+1)}&=\Omega_k-\eta_i^{(t)} \nabla l_i(\omega_i^0,D_i)-\dots-\eta_i^{(t)} \nabla l_i(\omega_i^{Q-1},D_i) \\
	\end{split}
\end{equation}

If $\eta_i^{(t)} \leq \frac{\|  \omega_i^{(t)}-\Omega_k\|_2}{Q U}$, we can prove that:
\begin{equation}
	\mathcal{F}^{(t,L)}\leq \mathcal{F}^{(t,M)}
\end{equation}
\end{lemma}

\begin{proof}
\begin{equation}
	\begin{split}
		\|  \omega_i^{(n+1)}-\Omega_k\|_2 &= \| \Omega_k- \eta_i^{(t)} \nabla l_i(\omega_i^0,D_i)-\dots-\eta_i^{(t)}\nabla l_i(\omega_i^{Q-1},D_i)-\Omega_k\|_2 \\
		&=\eta_i^{(t)} \| \nabla l_i(\omega_i^0,D_i)+\dots+\nabla l_i(\omega_i^{Q-1},D_i)\|_2
	\end{split}
\end{equation}
So if we want to:
\begin{equation}
	\begin{split}
		\|  \omega_i^{(n+1)}-\Omega_k\|_2^2 
		& =\eta_i^{(t)} \| \nabla l_i(\omega_i^0,D_i)+\dots+\nabla l_i(\omega_i^{Q-1},D_i)\|_2^2  \\
		&\leq (\eta_i^{(t)} Q U)^2 \\
		& \leq \|  \omega_i^{(t)}-\Omega_k\|_2^2 
	\end{split}
\end{equation}
$\eta$ should be:
\begin{equation}
	\eta_i^{(t)} \leq \frac{\|  \omega_i^{(t)}-\Omega_k\|_2}{Q U}
\end{equation}
In particular, if $\|  \omega_i^{(t)}-\Omega_k\|=0$, then $\eta_i^{(t)}=0$, $\omega_i$ does not change, or if $\| \nabla l_i\|$ equals 0, it means $\omega_i$ has been to the local minimum.
\end{proof}

For Theorem \ref{theo1}, the proof is as below:
\begin{proof}
In communication round t+1, use Lemma \ref{l1} \ref{l2} \ref{l3}, it is easy to get:
\begin{equation}
	\mathcal{F}^{(t+1,L)}\leq \mathcal{F}^{(t,L)}
\end{equation}
which also means $\mathcal{F}^{(t+1)} \leq \mathcal{F}^{(t)}$, because $\mathcal{F}$ must be non-negative, and there are finite steps for this minimization, then according to monotone convergence theorem for sequences, $\{ \mathcal{F}^{(t)}\}$ converges with finite iterations, which means for an arbitrary $\epsilon$, we can find a specific $N$, for any $ n>N, \mathcal{F}^{(t)}-\mathcal{F}^\star < \epsilon$.
\end{proof}

\subsection{Proof of Theorem \ref{theo2}}
\begin{lemma}
\label{lem4}
Under Assumption \ref{as_bound} and \ref{as1}, from the Expectation step to Maximization step in arbitrary communication round, $\mathcal{R}^{M} \leq \mathcal{R}^{E}+\eta B E U^2$.
\end{lemma}

\begin{proof}
\begin{align}
	\mathcal{R}^{M} - \mathcal{R}^{E} &= \frac{1}{\sum_{j=1}^m \lambda_j} \sum_{k=1}^{K}\sum_{i \in k} \lambda_i (\mathcal{L}(\Omega_k^{M}, D_i)-\mathcal{L}(\omega_i, D_i)) 
\end{align}
in which
\begin{align}
	\Omega_k^{M} = \sum_{p \in k} \frac{ \lambda_{p}}{\sum_{z \in k} \lambda_{z}} \omega_p
\end{align}
According to Assumption \ref{as1} and Equation \ref{eq:convex}, for arbitrary cluster, we have
\begin{align}
	&\sum_{i \in k} \lambda_i (\mathcal{L}(\sum_{p \in k} \frac{ \lambda_{p}}{\sum_{z \in k} \lambda_{z}} \omega_p, D_i)-\mathcal{L}(\omega_i, D_i)) \\ \leq& \sum_{i \in k} \lambda_i (\langle \nabla\mathcal{L}(\Omega_k^{M}, D_i), \sum_{p \in k} \frac{ \lambda_{p}}{\sum_{z \in k} \lambda_{z}} \omega_p - \omega_i \rangle) \\
	\leq& \sum_{i \in k} \lambda_i \|\nabla\mathcal{L}(\Omega_k^{M}, D_i)\|_2 \cdot \|\sum_{p \in k}\frac{ \lambda_{p}}{\sum_{z \in k} \lambda_{z}} \omega_p - \omega_i\|_2 (Cauchy–Schwarz)\\
	\leq& \sum_{i \in k} \lambda_i U \|\sum_{p \in k}\frac{ \lambda_{p}}{\sum_{z \in k} \lambda_{z}} \omega_p - \omega_i\|_2(Assumption\ \ref{as_bound})
\end{align}

According to Equation \ref{eq:wi},
\begin{align}
	\omega_i =\Omega_k-\eta \nabla l_i(\omega_i^0,D_i)-\dots-\eta \nabla l_i(\omega_i^{Q-1},D_i)
\end{align}
So we can get below inequality depending on Definition \ref{def:clu}:
\begin{align}
	\sum_{i \in k} \lambda_i U \|\sum_{p \in k}\frac{ \lambda_{p}}{\sum_{z \in k} \lambda_{z}} \omega_p - \omega_i\|_2 \leq \sum_{i \in k} \lambda_i \eta B Q U^2
\end{align}

Finally:
\begin{align}
	\mathcal{R}^{M} \leq \mathcal{R}^{E}+\eta B Q U^2
\end{align}
\end{proof}

\begin{lemma}
\label{lem5}
Under Assumption \ref{as2} and \ref{as3}, from the Maximization step to Local update step in arbitrary communication round, we have
\begin{align}
	{\mathbb E}[\mathcal{R}^{L}] - \mathcal{R}^{M} \leq \frac{1}{\sum_{j=1}^m \lambda_j} \sum_{k=1}^{K}\sum_{i \in k} \lambda_i \sum_{q=0}^{Q-1} ((\frac{\beta \eta^2_q}{2} - \eta_q) {\mathbb E}[\| \nabla \mathcal{L}(\Omega_k^{(M,q)})\|_2^2] + \frac{\beta \eta^2_q}{2} \sigma^2)
\end{align}

\end{lemma}

\begin{proof}
\begin{align}
	\mathcal{R}^{L} - \mathcal{R}^{M} &= \frac{1}{\sum_{j=1}^m \lambda_j} \sum_{k=1}^{K}\sum_{i \in k} \lambda_i (\mathcal{L}(\Omega_k^{L}, D_i)-\mathcal{L}(\Omega_k^{M}, D_i)) 
\end{align}
For arbitrary Client i, using Gradient Descent,
\begin{align}
	\mathcal{L}(\Omega_k^{L}, D_i)-\mathcal{L}(\Omega_k^{M}, D_i) &= \sum_{q=0}^{Q-1} (\mathcal{L}(\Omega_k^{(M,q+1)}, D_i) -\mathcal{L}(\Omega_k^{(M,q)}, D_i))
\end{align}
Under Assumption \ref{as2},
\begin{align}
	\mathcal{L}(\Omega_k^{(M,q+1)}) -\mathcal{L}(\Omega_k^{(M,q)}) & \leq \langle \nabla \mathcal{L}(\Omega_k^{(M,q)}), \Omega_k^{(M,q+1)}  -\Omega_k^{(M,q)}\rangle + \frac{\beta}{2} \| \Omega_k^{(M,q+1)}  -\Omega_k^{(M,q)}\|_2^2 \\
	& = -\eta \langle \nabla \mathcal{L}(\Omega_k^{(M,q)}), \nabla \mathcal{L}(\Omega_k^{(M,q)}, \xi^e_i) \rangle + \frac{\beta \eta^2}{2} \| \nabla \mathcal{L}(\Omega_k^{(M,q)}, \xi^e_i)\|_2^2 
\end{align}
take expectation on both sides for random selected batch $\xi^e_i$ under Assumption \ref{as3}, 
\begin{align}
	{\mathbb E}[\mathcal{L}(\Omega_k^{(M,q+1)})] -\mathcal{L}(\Omega_k^{(M,q)}) \leq (\frac{\beta \eta^2}{2} - \eta)\| \nabla \mathcal{L}(\Omega_k^{(M,q)})\|_2^2 + \frac{\beta \eta^2}{2} \sigma^2
\end{align}

\noindent take expectation on both sides again on random variable $\Omega_k^{(M,q)}$, and do telesoping, we can get, 
\begin{align}
	{\mathbb E}[\mathcal{L}(\Omega_k^{L}, D_i)]-\mathcal{L}(\Omega_k^{M}, D_i) &= \sum_{q=0}^{Q-1} ({\mathbb E}[\mathcal{L}(\Omega_k^{(M,q+1)}, D_i)] -\mathcal{L}(\Omega_k^{(M,q)}, D_i)) \\
	& \leq \sum_{q=0}^{Q-1} ((\frac{\beta \eta^2_q}{2} - \eta_q) {\mathbb E}[\| \nabla \mathcal{L}(\Omega_k^{(M,q)})\|_2^2] + \frac{\beta \eta^2_q}{2} \sigma^2)
\end{align}
Finally,
\begin{align}
	{\mathbb E}[\mathcal{R}^{L}] - \mathcal{R}^{M} \leq \frac{1}{\sum_{j=1}^m \lambda_j} \sum_{k=1}^{K}\sum_{i \in k} \lambda_i \sum_{q=0}^{Q-1} ((\frac{\beta \eta^2_q}{2} - \eta_q) {\mathbb E}[\| \nabla \mathcal{L}(\Omega_k^{(M,q)})\|_2^2] + \frac{\beta \eta^2_q}{2} \sigma^2)
\end{align}
\end{proof}

Then for Theorem \ref{theo2}, the proof is as below:
\begin{proof}
From the local distribution step in communication round t-1 to the Expecation step in communication round t, what is changed in loss function of WeCFL $\mathcal{R}$ is the $r_i^k$, but the $\mathcal{L}(\Omega_k,D_i)$ does not change, so we can get
\begin{align}
	\mathcal{R}^{(t-1,L)} = \mathcal{R}^{(t,E)}
\end{align}
then according to Lemma \ref{lem4} and \ref{lem5}, we can get,
\begin{align}
	&{\mathbb E}[\mathcal{R}^{(t,L)}]-\mathcal{R}^{(t-1,L)} \\ &\leq \eta B Q U^2+\frac{1}{\sum_{j=1}^m \lambda_j} \sum_{k=1}^{K}\sum_{i \in k} \lambda_i \sum_{q=0}^{Q-1} ((\frac{\beta \eta^2_{(t,q)}}{2} -\eta_{(t,q)}) {\mathbb E}[\| \nabla \mathcal{L}(\Omega_k^{(t,M,q)})\|_2^2] +\frac{\beta \eta^2_{(t,q)}}{2} \sigma^2)\\
	& = \frac{1}{\sum_{j=1}^m \lambda_j} \sum_{k=1}^{K}\sum_{i \in k} \lambda_i \sum_{q=0}^{Q-1} ((\frac{\beta \eta^2_{(t,q)}}{2} -\eta_{(t,q)}) {\mathbb E}[\| \nabla \mathcal{L}(\Omega_k^{(t,M,q)})\|_2^2] +\frac{\beta \eta^2_{(t,q)}}{2} \sigma^2+ \eta_{(t,q)} B U^2)
	\label{eq 35}
\end{align}
then when 
\begin{align}
	\eta_{(t,q)} < min\{\frac{\| \omega_i^{(t)}-\Omega_k\|}{Q U},\frac{{\mathbb E}[\| \nabla \mathcal{L}(\Omega_k^{(t,M,q)})\|_2^2] -B U^2}{{\mathbb E}[\| \nabla \mathcal{L}(\Omega_k^{(t,M,q)})\|_2^2] +\sigma^2} \cdot \frac{2}{\beta}\}
\end{align}
the right term of Equation \ref{eq 35} is always negative. So we can ensure that the EM loss function $\mathcal{F}$ converges, and the FL loss function $\mathcal{R}$ decreases monotonically, thus the WeCFL converges.
\end{proof}

\subsection{Proof of Theorem \ref{theo3}}
\begin{proof}
Take expectation of Equation \ref{eq 35} on the parameter, then do telescoping from 0 to $T$, we can get,
\begin{align}
	&\Delta \geq \mathcal{R}^{(0,L)} - {\mathbb E}[\mathcal{R}^{(T,L)}]  \\ &\geq  \sum_{k=1}^{K}\sum_{i \in k} \sum_{t=0}^{T-1} \sum_{q=0}^{Q-1} \frac{\lambda_i }{\sum_{j=1}^m \lambda_j} ((\eta_{(t,q)} -\frac{\beta \eta^2_{(t,q)}}{2} ) {\mathbb E}[\| \nabla \mathcal{L}(\Omega_k^{(t,M,q)})\|_2^2] - \frac{\beta \eta^2_{(t,q)}}{2} \sigma^2-\eta_{(t,q)} U^2)
\end{align}
if
\begin{align}
	\frac{1}{ T Q}\sum_{k=1}^{K}\sum_{i \in k} \sum_{t=0}^{T-1} \sum_{q=0}^{Q-1} \frac{\lambda_i }{\sum_{j=1}^m \lambda_j} {\mathbb E}[\| \nabla \mathcal{L}(\Omega_k^{(t,M,q)})\|_2^2] \leq \epsilon
\end{align}
then
\begin{align}
	T \geq \frac{\Delta}{Q (\epsilon (\eta - \frac{\beta \eta^2}{2})-\frac{\beta \eta^2}{2} \sigma^2-\eta B U^2)}
\end{align}
\end{proof}

\pagebreak{}
\section{More Deatils of Experiments} \label{app:exp}

\subsection{More about experimental settings}

\textbf{Optimization settings} For the training model, we use small CNNs with two convolutional layers for Fashion-MNIST and CIFAR-10 as shown in Table \ref{tab:cnn_mnist} and \ref{tab:cnn_cifar}, respectively. For the optimization, SGD with the learning rate 0.001 and momentum 0.9 is used to train the model, and the batch size is 32.

\begin{table}[h]
\caption{Detailed information of the CNN for Fashion-MNIST.}
\label{tab:cnn_mnist}
\begin{center}
	\small
		\begin{tabular}{c|l}
			\toprule
			Layer & Details \\
			\midrule
			\multirow{4}{*}{Convolution} & $Conv2d(1,16,kernel\_size = (5,5), padding =2)$ \\
			&$BatchNorm2d(16)$\\
			& $ReLU()$\\
			& $MaxPool2d(2,2)$ \\
			\midrule
			\multirow{4}{*}{Convolution} &  $Conv2d(16,32, kernel\_size = (5,5), padding =2)$\\
			&$BatchNorm2d(16)$\\
			& $ReLU()$\\
			& $MaxPool2d(2,2)$ \\
			\midrule
			Classifier & $Linear(7*7*32,10)$ \\
			\midrule
			Loss & $CrossEntropy()$ \\
			\bottomrule
		\end{tabular}
\end{center}
\end{table}

\begin{table}[h]
\caption{Detailed information of the CNN for CIFAR-10.}
\label{tab:cnn_cifar}
\begin{center}
	\small
		\begin{tabular}{c|l}
			\toprule
			Layer & Details \\
			\midrule
			\multirow{3}{*}{Convolution}  & $Conv2d(3,6,kernel\_size = (5,5))$ \\
			& $ReLU()$\\
			& $MaxPool2d(2,2)$ \\
			\midrule
			\multirow{3}{*}{Convolution}  &  $Conv2d(6,16,kernel\_size = (5,5))$\\
			& $ReLU()$\\
			& $MaxPool2d(2,2)$ \\
			\midrule
			\multirow{2}{*}{Linear} &  $Linear(400,120)$\\
			& $ReLU()$\\
			\midrule
			\multirow{2}{*}{Linear} &  $Linear(120,84)$\\
			& $ReLU()$\\
			\midrule
			Classifier & $Linear(84,10)$ \\
			\midrule
			Loss & $CrossEntropy()$ \\
			\bottomrule
		\end{tabular}
\end{center}
\end{table}

\textbf{FL settings} For the FL settings, the non-IID pre-processing visualization is shown in Figure \ref{fig:4noniid}. We run 100 global communication rounds, and the local steps in each communication are 10. For the clustering process, we use flattened parameters of the fully-connected layers of CNNs as data points and weighted K-Means as the clustering algorithm.

\begin{figure}
\vspace{-0.4cm}
\centering
\begin{minipage}{0.48\textwidth}
	\centering
	\includegraphics[width=1\textwidth]{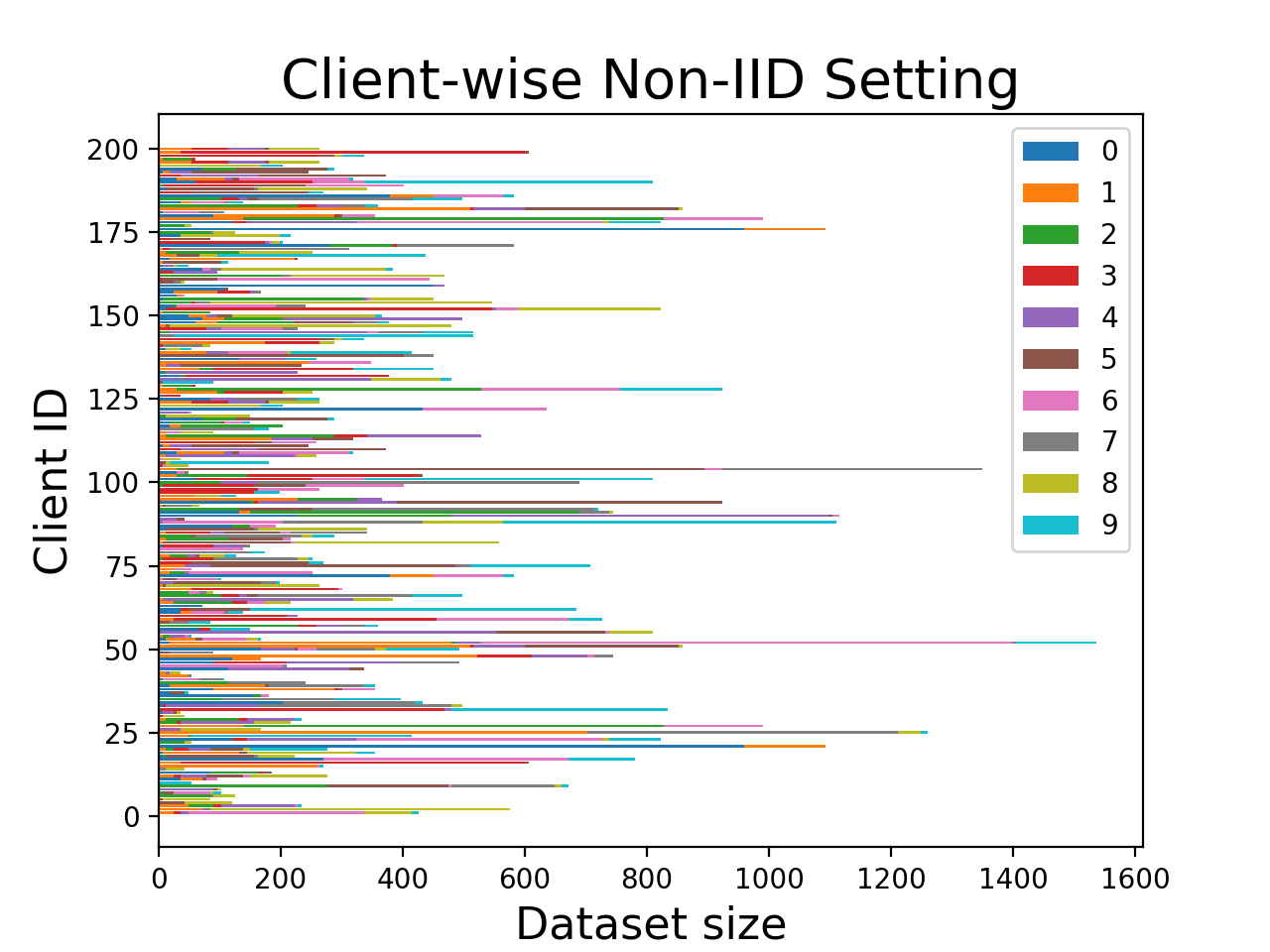} 
\end{minipage}\hfill
\begin{minipage}{0.48\textwidth}
	\centering
	\includegraphics[width=1\textwidth]{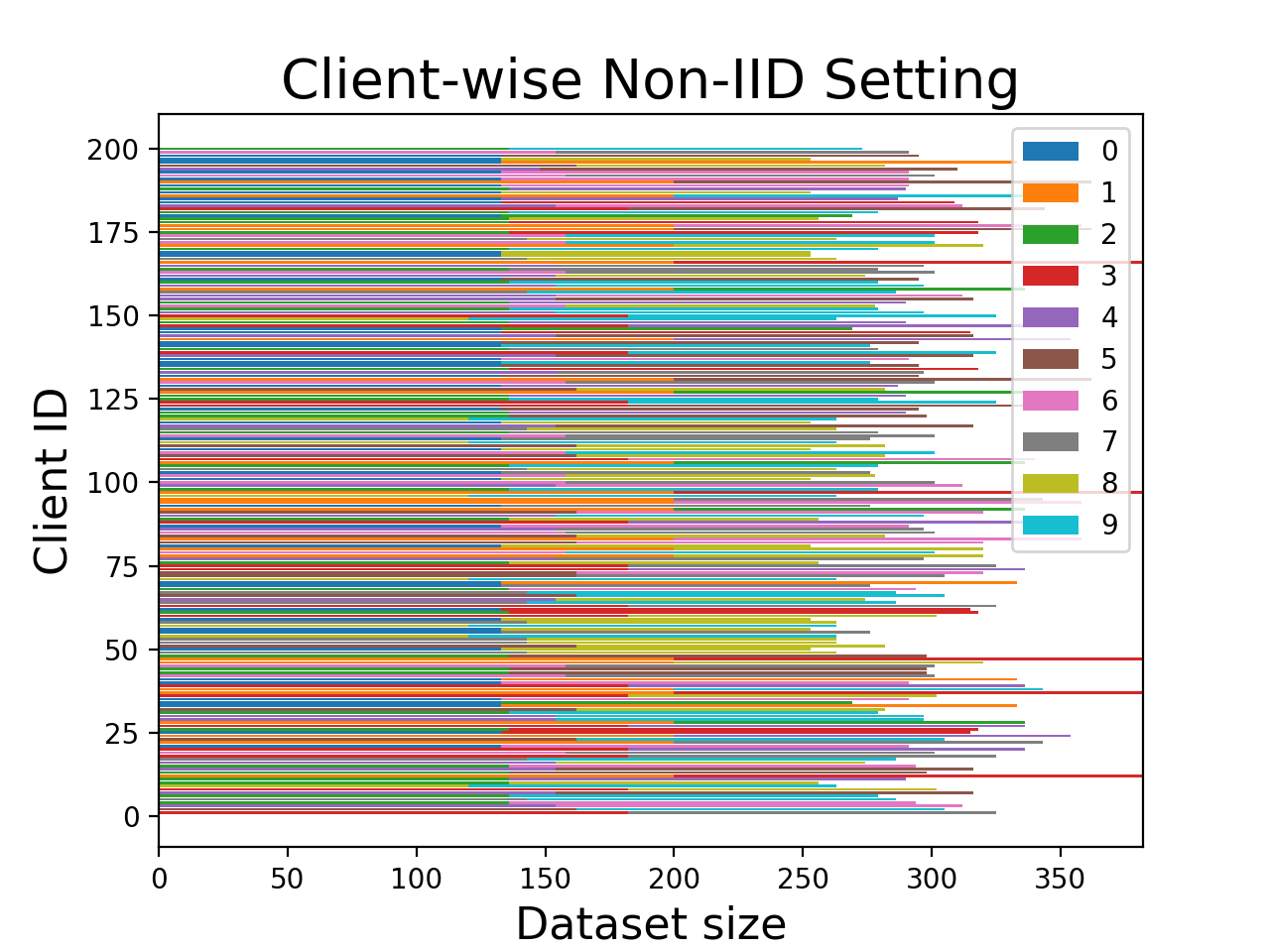} 
\end{minipage} \vfill
\begin{minipage}{0.48\textwidth}
	\centering
	\includegraphics[width=1\textwidth]{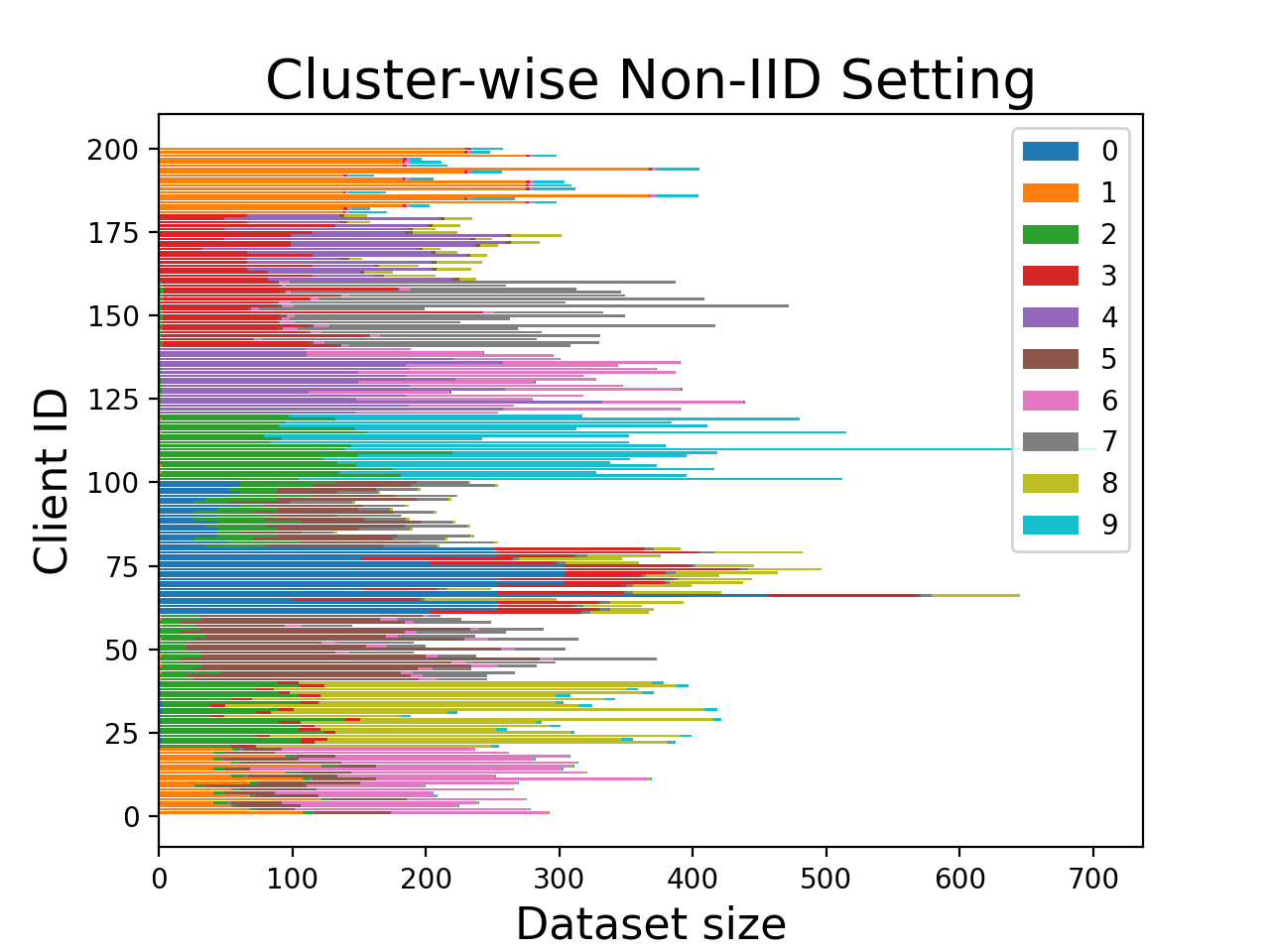} 
\end{minipage}\hfill
\begin{minipage}{0.48\textwidth}
	\centering
	\includegraphics[width=1\textwidth]{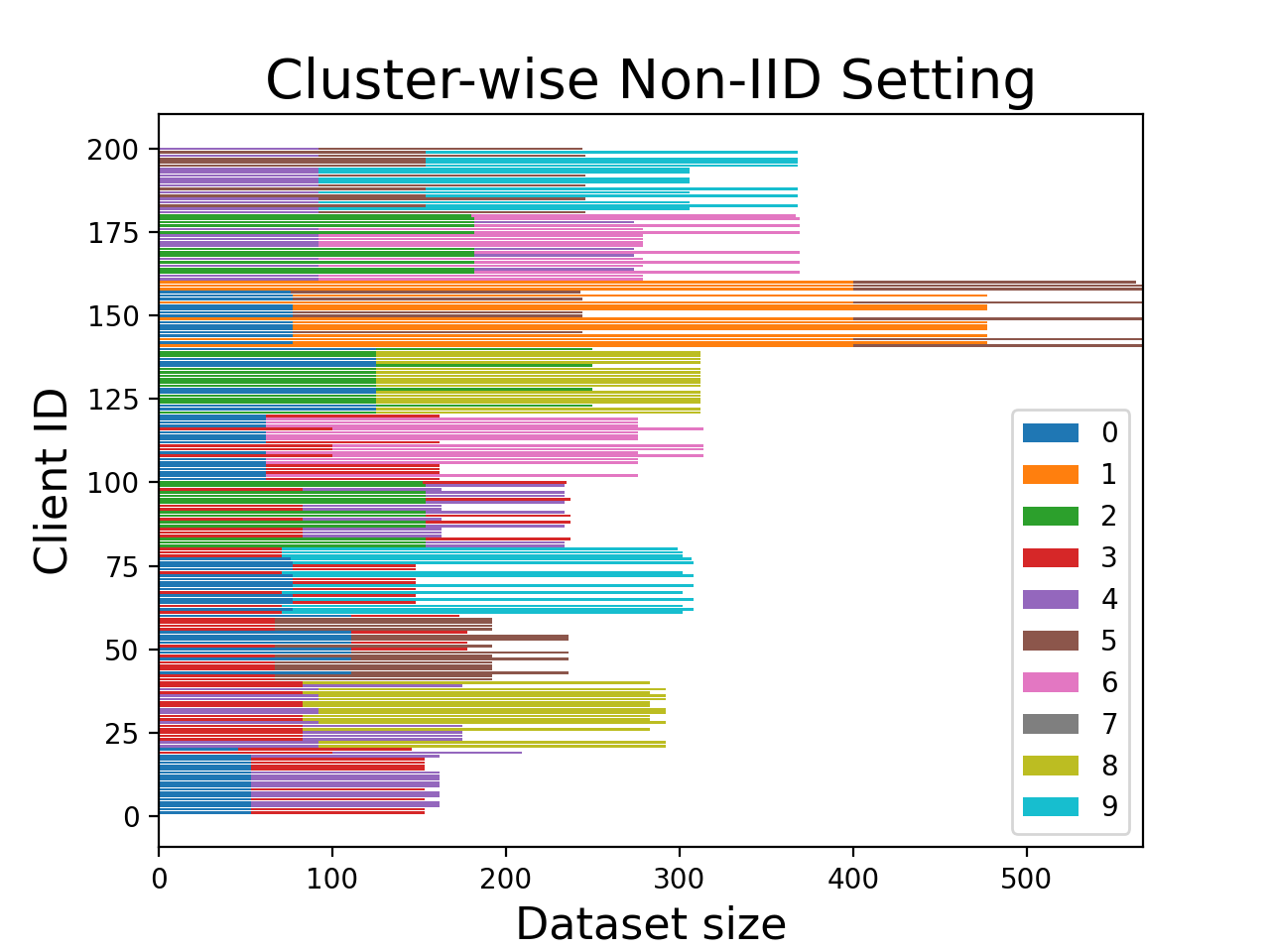} 
\end{minipage}
\caption{\small An example visualization of all four non-IID pre-processing methods on the Fashion-MNIST according to our experimental settings.}
\vspace{-0.2cm}
\label{fig:4noniid}
\end{figure}




\subsection{More about experimental results}

\textbf{Client-wise non-IID results} The experiment results in the client-wise non-IID setting is shown below in Table \ref{tab:client}. The results show that WeCFL outperforms almost all baselines. The statistical heterogeneity of CIFAR-10 is much higher than Fashion-MNIST or other MNIST dataset families. Therefore, WeCFL demonstrates superior performance improvements in CIFAR-10 than in Fashion-MNIST. Within a proper interval, larger $K$ leads to better performance. As shown in the table, when $K$ is increased from 5 to 10, all methods' performance increases. Furthermore, with a higher $K$, the performance of WeCFL improves more in CIFAR-10 than in Fashion-MNIST.

\begin{table}[H]
	\vspace{-0.3cm}
	\caption{Performance comparison on client-wise non-IID }
	\label{tab:client}
	\begin{center}
		\resizebox{1.0\columnwidth}{!}{%
			\begin{tabular}{ll|cccc|cccc}
				\toprule
				\multicolumn{2}{c}{Datasets} & \multicolumn{4}{c}{Fashion-MNIST} & \multicolumn{4}{c}{CIFAR-10} \\
				
				\midrule
				\midrule
				
				\multicolumn{2}{c}{Non-IID setting} & \multicolumn{2}{c}{$\alpha=0.1$} & \multicolumn{2}{c}{$2-$class} & \multicolumn{2}{c}{$\alpha=0.1$} & \multicolumn{2}{c}{$2-$class}  \\
				
				\midrule
				
				\textbf{K} & Methods & Accuracy & Macro-F1 & Accuracy & Macro-F1  & Accuracy & Macro-F1 & Accuracy & Macro-F1 \\
				
				\midrule
				
				\multirow{2}{*}{\textbf{1}} & FedAvg    &   85.9$\pm$0.46 & 54.52$\pm$2.66   & 86.17$\pm$0.25 & 44.88$\pm$1.24   &  25.62$\pm$3.47 & 11.38$\pm$2.02   &  24.3$\pm$3.53 & 8.56$\pm$0.64    \\
				& FedProx    & 86.03$\pm$0.58 & 54.69$\pm$3.32 & 86.47$\pm$0.23 &44.89$\pm$1.38 & 25.72$\pm$3.29 & 11.14$\pm$1.49  &  24.19$\pm$2.45 & 8.69$\pm$0.74\\
				
				\midrule
				\multirow{5}{*}{\textbf{5}} & FedAvg+    &  86.12 & 61.07    & 86.5   & 45.39   &  25.71  & 12.45  &         24.83  &  8.74             \\ 
				
				& FedProx+ & 86.39 &  56.56  & 86.15  &  45.43   &  25.58  &  12.43  &   25.88  &  8.55       \\ 
				
				& IFCA &  90.13$\pm$6.81 & 68.47$\pm$5.23 & 91.54$\pm$5.04 & 72.3$\pm$5.32 & 47.21$\pm$ 10.28 & 22.67$\pm$1.48 &  46.54$\pm$12.8 & 17.78$\pm$1.29 \\ 
				
				& FeSEM &  91.51$\pm$2.9 & 73.78$\pm$9.88 & 91.83$\pm$1.24 & 71.05$\pm$8.63 & 54.3$\pm$4.58 & 24.78$\pm$6.01 &  55.55$\pm$4.83 & 32.8$\pm$4.18 \\
				& WeCFL &  91.59$\pm$0.82 & 74.45$\pm$10.53  &      91.76$\pm$1.53  &  69.47$\pm$5.04  & 55.09$\pm$5.1 & 27.29$\pm$8.37 &        55.89$\pm$5.92  &   33.12$\pm$5.0 \\ 
				
				\midrule
				\multirow{5}{*}{\textbf{10}} & FedAvg+    &   86.81 & 60.43    & 86.91    &  47.12     &   27.83  & 13.65  &       27.71  &  9.65         \\ 
				
				& FedProx+ &  86.24 &  56.2  &    86.78  &  42.83   &  25.86  &  12.84    &   26.16  &  9.94      \\ 
				
				& IFCA &  91.04$\pm$4.33 & 68.6$\pm$6.77 & 91.42$\pm$5.16 & 72.29$\pm$5.8& 47.62$\pm$10.15 & 23.36$\pm$2.48&  47.96$\pm$10.59 &17.88$\pm$1.04 \\
				& FeSEM &  93.3$\pm$2.0 & \textbf{80.47$\pm$11.05} & 93.75$\pm$1.53 &79.39$\pm$6.57 & 67$\pm$1.57 &31.69$\pm$8.52 &  63.64$\pm$6.51  & 42.97$\pm$6.08\\
				& WeCFL &   \textbf{94.21$\pm$1.67} & 79.31$\pm$11.02 & \textbf{94.05$\pm$1.67} &  \textbf{81.41$\pm$5.7} & \textbf{69.47$\pm$4.16} & \textbf{34.1$\pm$7.79}  & \textbf{66.8$\pm$6.39} & \textbf{45.61$\pm$5.9}\\
				\bottomrule
			\end{tabular}
		}
	\end{center}
	\vspace{-0.5cm}
\end{table}



\textbf{Clustering visualization} Figure \ref{fig:tsne_10} and \ref{fig:tsne_3} demonstrats the changing clustering results in view of t-SNE for the first five communication rounds on the Fashion-MNIST for $K=10$ and $K=3$, respectively, while the non-IID setting is $\alpha=(0.1,10)$ cluster-wise and the ground truth of cluster number $K$ is ten. Then, the clustering analysis of WeCFL can be summarized below,
\begin{itemize}
	\item The clustering converges very fast. For $K=10$, it takes only one communication round to converge. Even for $K=3$, it takes only three communication rounds to converge. With more communications, the inter-cluster distance becomes larger and intra-cluster distance becomes smaller.
	\item The clustering converges very well. For $K=10$, the clustering results exactly match the initial partition or ground truth. For $K=3$ that can not divide 10, the clustering results keep the initialized clusters and no break up.
	\item The range of the clusters or intra-cluster distance becomes smaller and smaller by the communication round for $K=10$ and $K=3$, which indicates that the clusterability measure $B$ is better and better. 
\end{itemize}


\begin{figure}
	\vspace{-0.4cm}
	\centering
	\begin{minipage}{0.48\textwidth}
		\centering
		\includegraphics[width=1\textwidth]{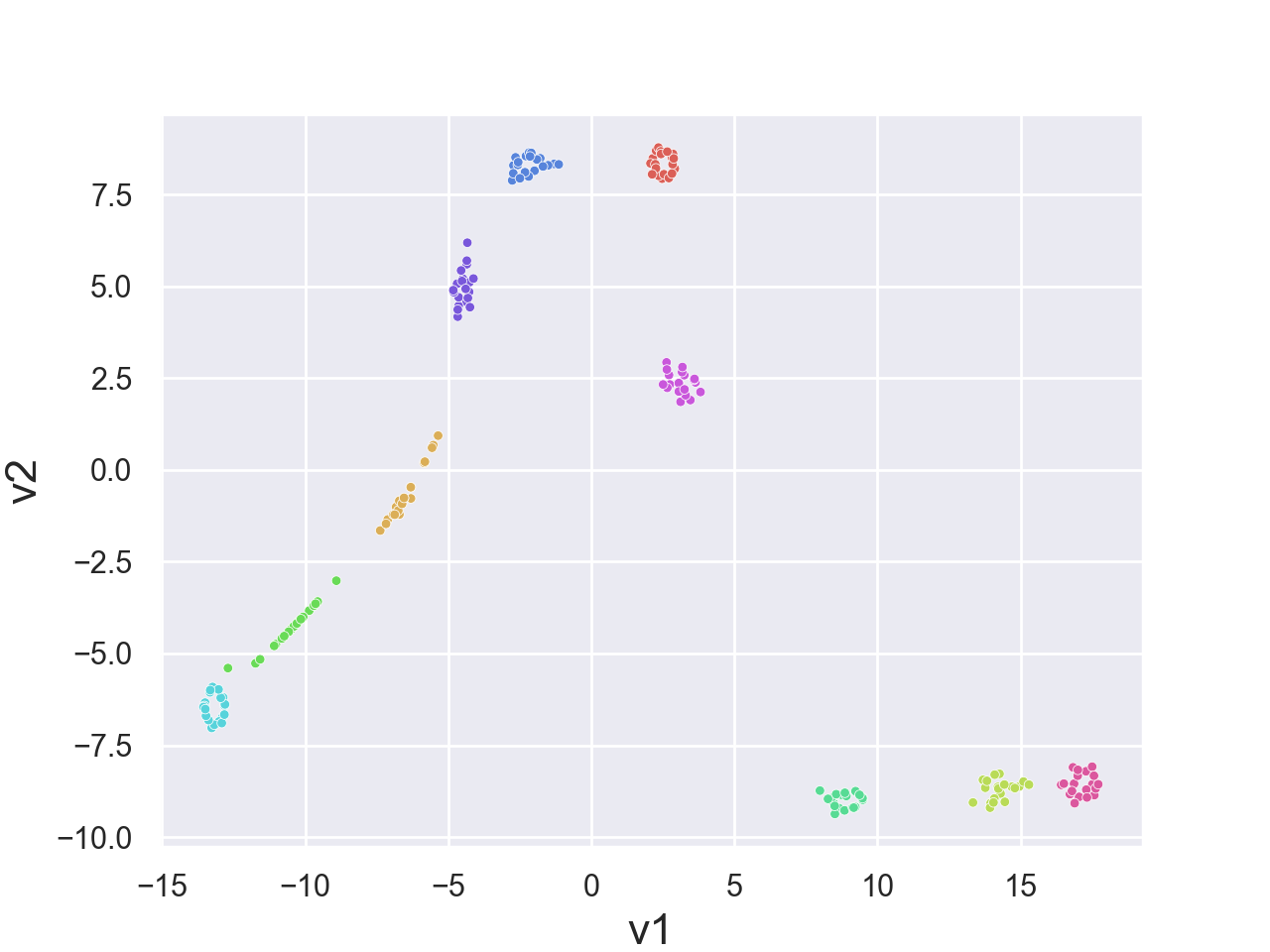} 
		\vspace{-0.3cm}
	\end{minipage}\hfill
	\begin{minipage}{0.48\textwidth}
		\centering
		\includegraphics[width=1\textwidth]{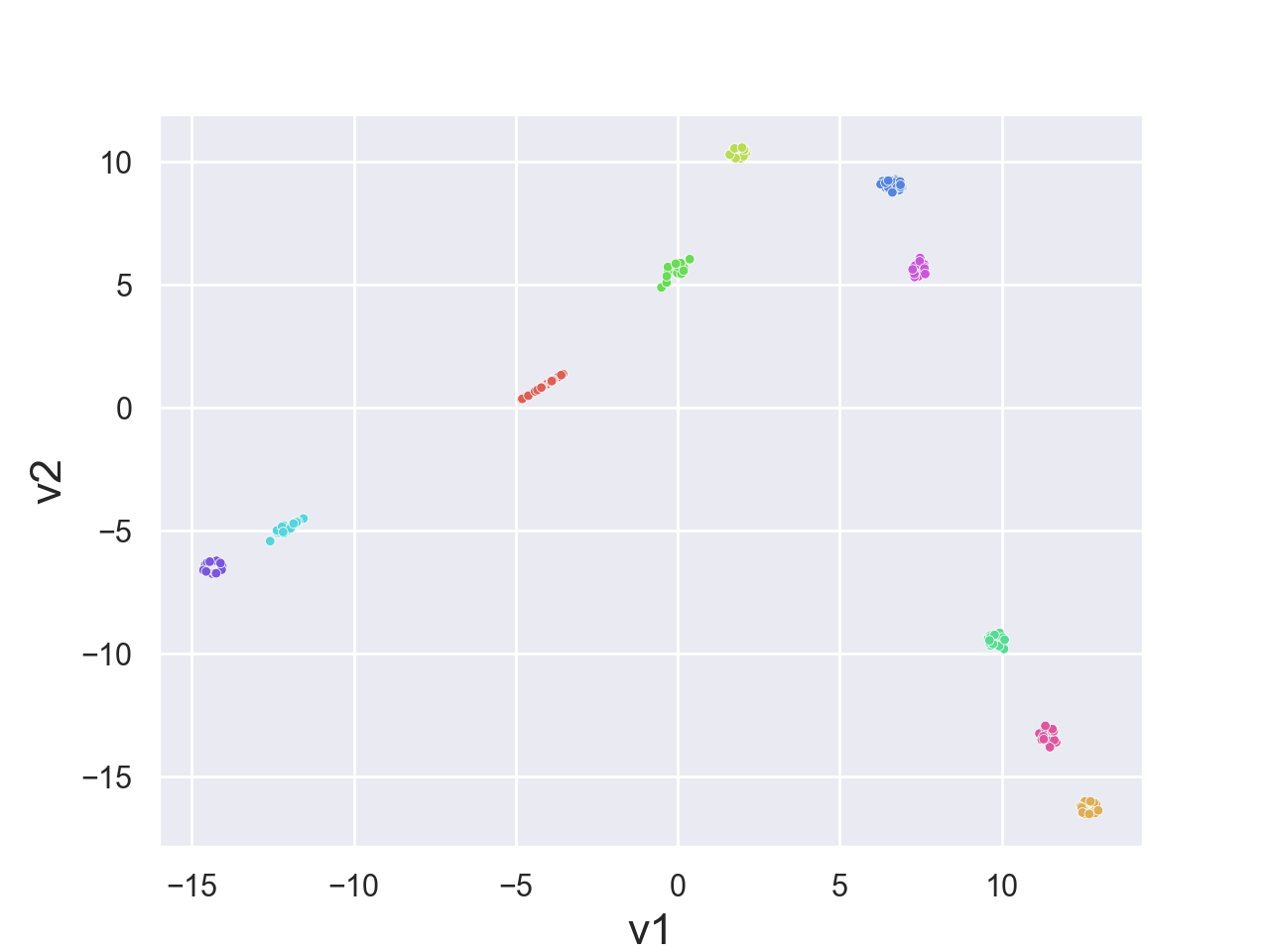} 
		\vspace{-0.3cm}
	\end{minipage} \vfill
	\begin{minipage}{0.48\textwidth}
		\centering
		\includegraphics[width=1\textwidth]{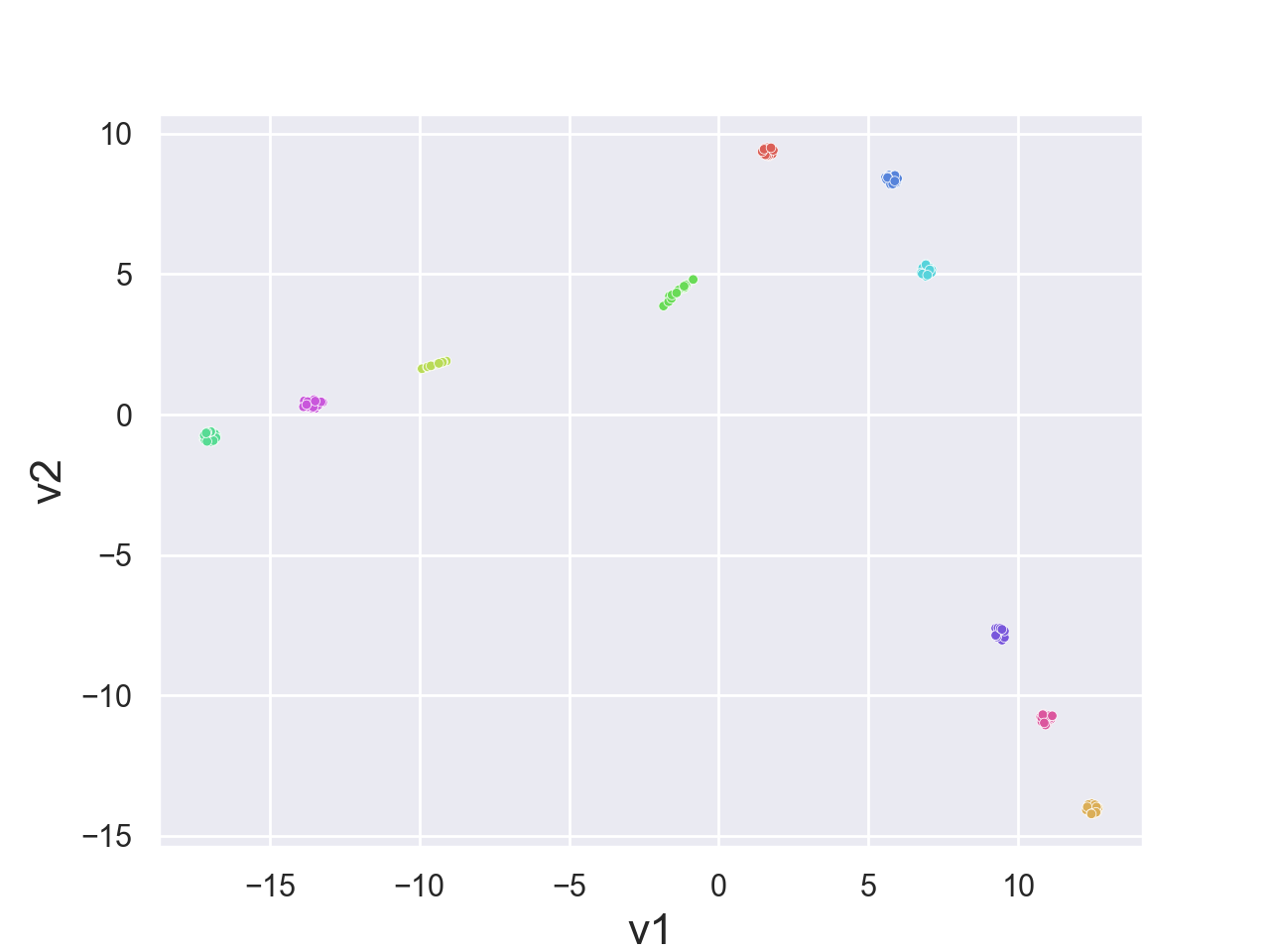} 
		\vspace{-0.3cm}
	\end{minipage}\hfill
	\begin{minipage}{0.48\textwidth}
		\centering
		\includegraphics[width=1\textwidth]{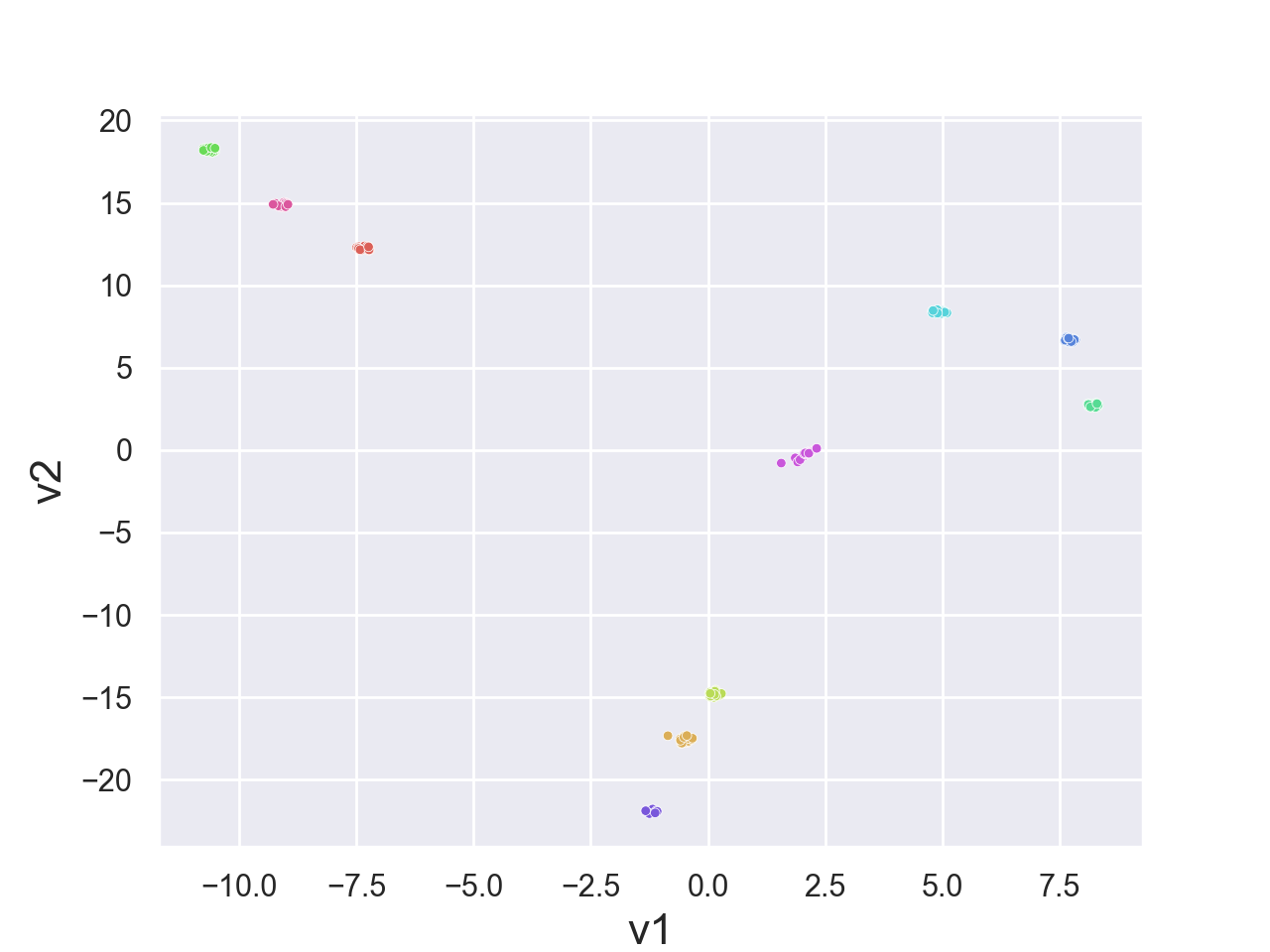} 
		\vspace{-0.3cm}
	\end{minipage}\vfill
	\begin{minipage}{0.48\textwidth}
		\centering
		\includegraphics[width=1\textwidth]{figs/tsne_f.png} 
		\vspace{-0.3cm}
	\end{minipage}
	\caption{\small T-SNE visualization of clustering results on the Fashion-MNIST in the first five communication rounds under the $\alpha=(0.1,10)$ cluster-wise non-IID setting, generated by 200 clients across $K=10$ clusters. Different colors represent different cluster labels. The order is left-to-right then top-to-bottom.}
	\vspace{-0.5cm}
	\label{fig:tsne_10}
\end{figure}

\begin{figure}
	\vspace{-0.4cm}
	\centering
	\begin{minipage}{0.48\textwidth}
		\centering
		\includegraphics[width=1\textwidth]{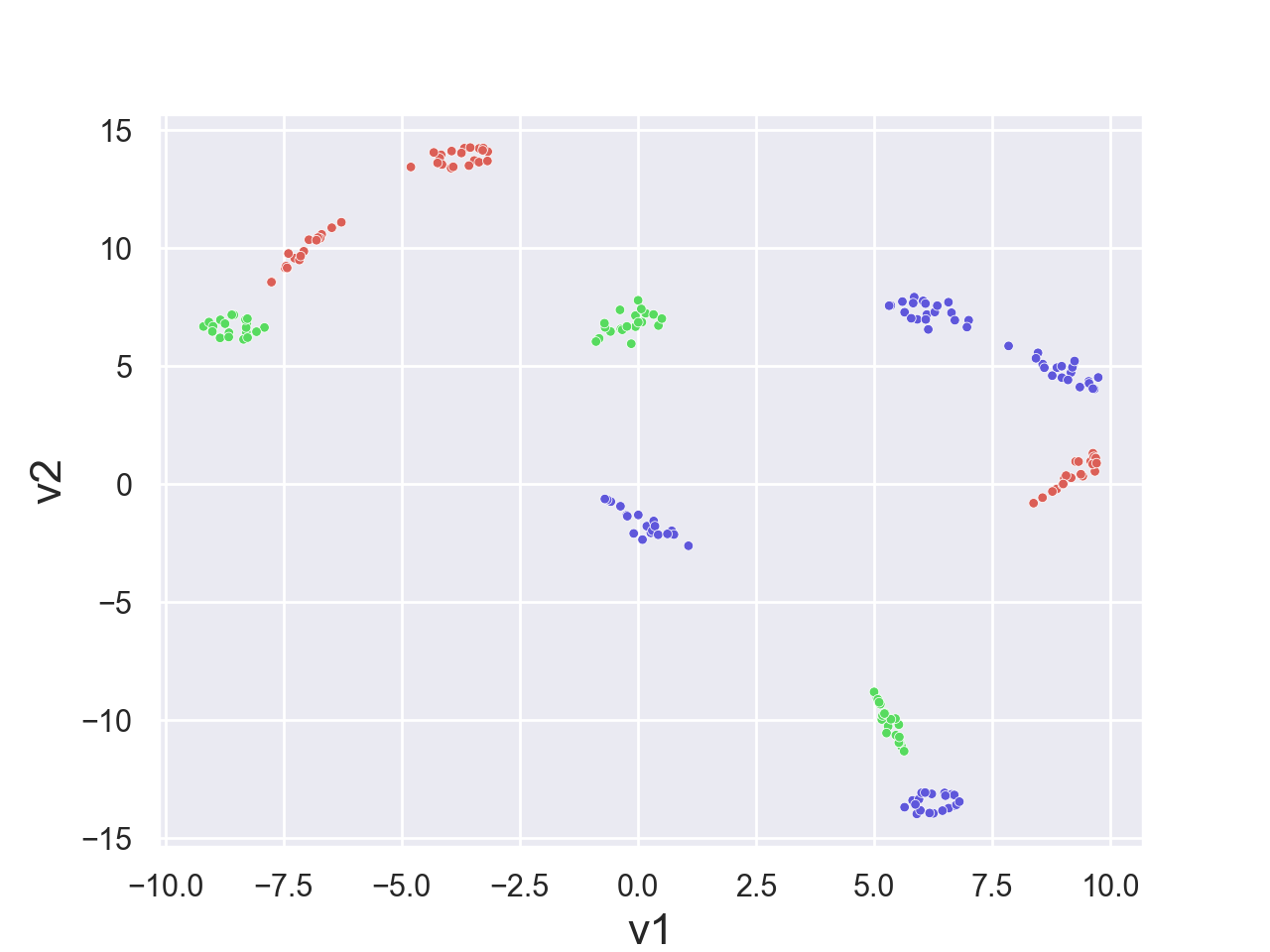} 
		\vspace{-0.3cm}
	\end{minipage}\hfill
	\begin{minipage}{0.48\textwidth}
		\centering
		\includegraphics[width=1\textwidth]{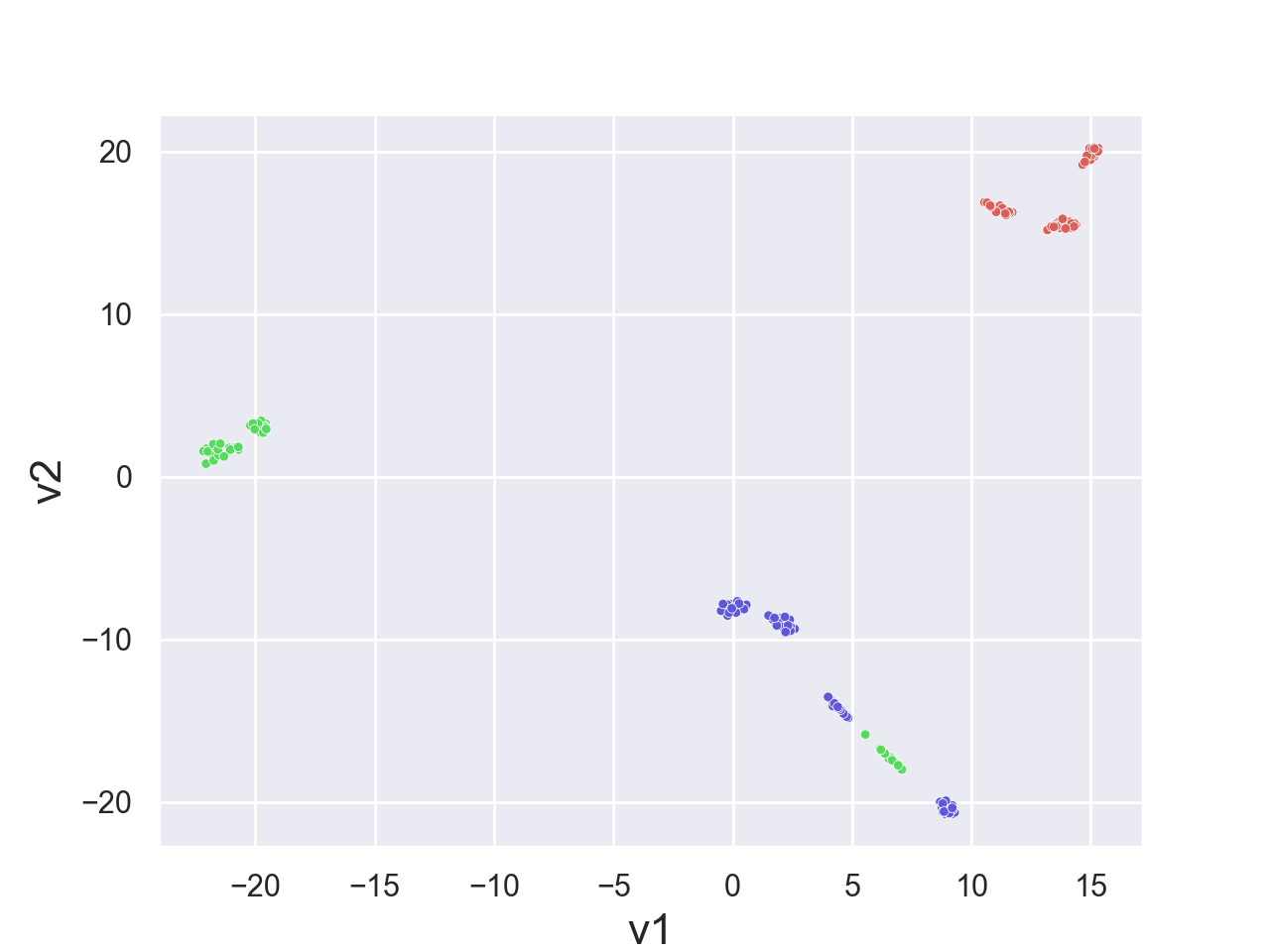} 
		\vspace{-0.3cm}
	\end{minipage} \vfill
	\begin{minipage}{0.48\textwidth}
		\centering
		\includegraphics[width=1\textwidth]{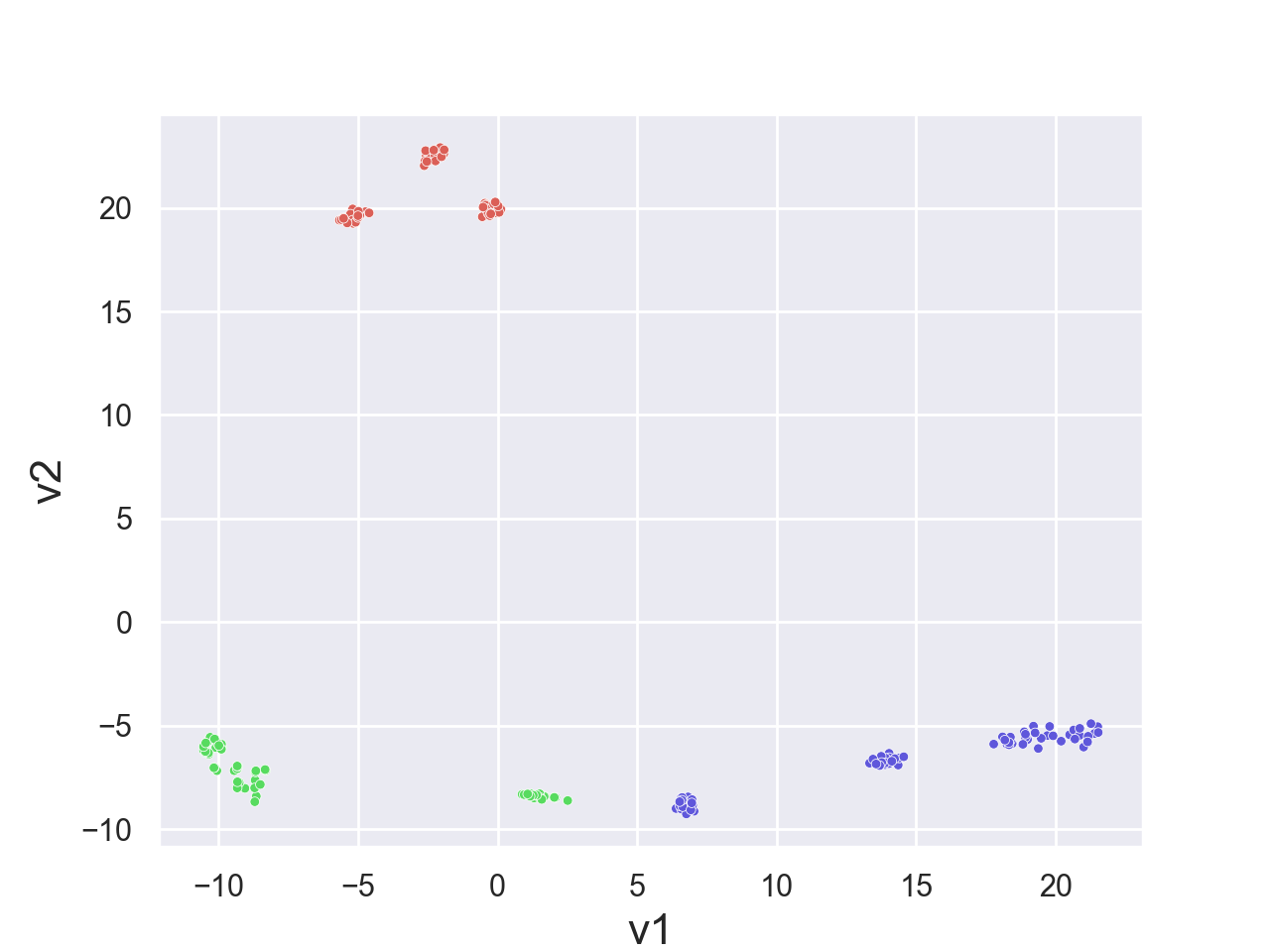} 
		\vspace{-0.3cm}
	\end{minipage}\hfill
	\begin{minipage}{0.48\textwidth}
		\centering
		\includegraphics[width=1\textwidth]{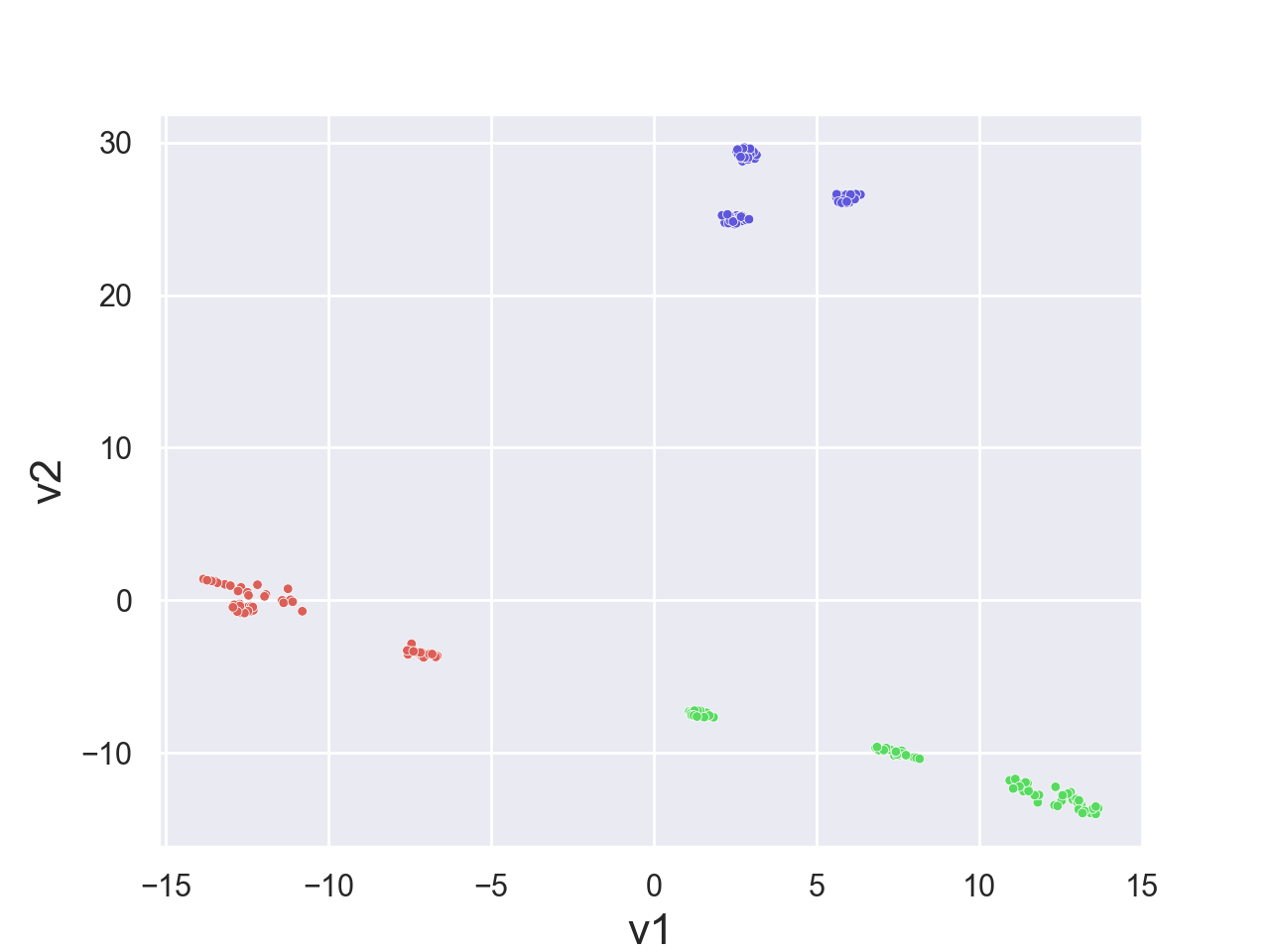} 
		\vspace{-0.3cm}
	\end{minipage}\vfill
	\begin{minipage}{0.48\textwidth}
		\centering
		\includegraphics[width=1\textwidth]{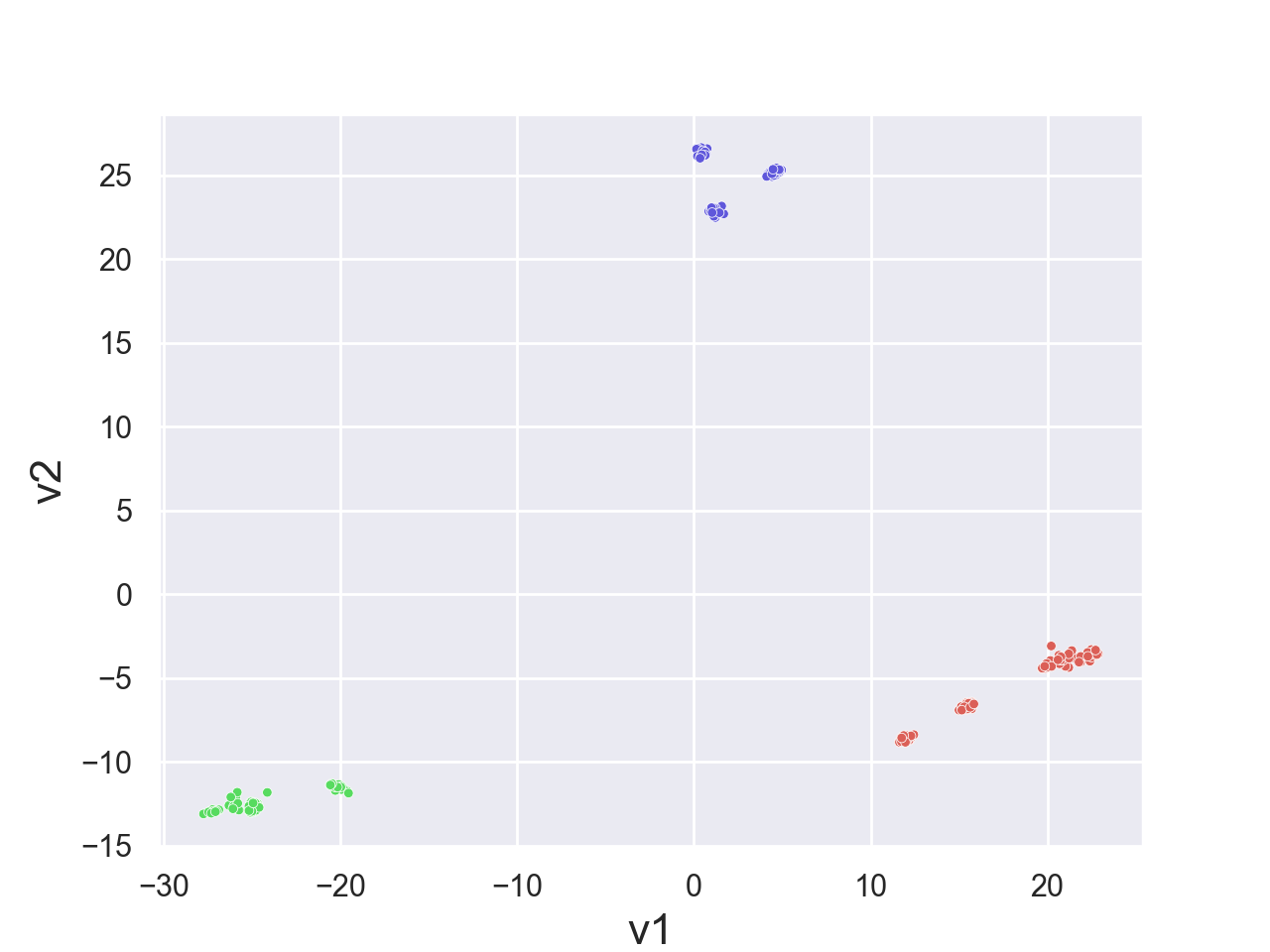} 
		\vspace{-0.3cm}
	\end{minipage}
	\caption{\small T-SNE visualization of clustering results in the first five communication rounds on the Fashion-MNIST under the $\alpha=(0.1,10)$ cluster-wise non-IID setting, generated by 200 clients across $K=3$ clusters. Different colors represent different cluster labels. The order is left-to-right then top-to-bottom.}
	\vspace{-0.5cm}
	\label{fig:tsne_3}
\end{figure}


\end{document}